\documentclass{article}

\usepackage{microtype}
\usepackage{graphicx}
\usepackage{hyperref}
\usepackage{subcaption}
\usepackage{booktabs} % for professional tables
\usepackage{multirow}
\usepackage{hyperref}
\usepackage{xspace}

\usepackage[accepted]{icml2026}
\usepackage{amsmath}
\usepackage{amssymb}
\usepackage{mathtools}
\usepackage{amsthm}
\usepackage{subcaption} % for subfigure environment
\usepackage{hyperref}       % hyperlinks
\usepackage{url}            % simple URL typesetting
\usepackage{booktabs}       % professional-quality tables
\usepackage{amsfonts}       % blackboard math symbols
\usepackage{nicefrac}       % compact symbols for 1/2, etc.
\usepackage{enumitem}
\usepackage{microtype}      % microtypography
\usepackage{graphicx} 
\usepackage{xcolor}         % colors
 \usepackage{amsmath} 
\newcommand{\alg}{\text{PTA-GRPO}\xspace}

\usepackage{amsthm}
\newcommand{\myparatights}[1]{\smallskip\noindent{\bf {#1}.}~}
\DeclareMathOperator{\E}{\mathbb{E}}

\usepackage{float} % 导言区

\usepackage[table]{xcolor}

\theoremstyle{plain}

\usepackage{cuted}
\usepackage{balance} 

\usepackage[capitalize,noabbrev]{cleveref}

\theoremstyle{plain}
\newtheorem{theorem}{Theorem}[section]
\newtheorem{proposition}[theorem]{Proposition}

\theoremstyle{definition}

\theoremstyle{remark}
\newtheorem{remark}[theorem]{Remark}
\definecolor{cvprblue}{rgb}{0.21,0.49,0.74}
\definecolor{greyL}{RGB}{230,248,255}
\definecolor{yellow}{RGB}{255,246,234}
\definecolor{green}{RGB}{230,247,239}
\definecolor{blues}{RGB}{226,234,249}

\usepackage[most]{tcolorbox}
\usepackage{amsmath,amssymb}
\newtcolorbox{promptpbox}[2][]{%
  enhanced,
  colback=yellow,          % 背景改为浅灰蓝
  colframe=blues,      % 边框改为深蓝
  coltitle=cvprblue,
  fonttitle=\bfseries\large,
  boxrule=1pt,            % 稍微加粗边框
  arc=3mm,                % 圆角
  left=6pt,
  right=6pt,
  top=6pt,
  bottom=6pt,
  title=#2,
  #1
}

\newtcolorbox{rlpbox}[1][]{%
  enhanced,
  breakable,
  verbatim,
  colback=green!5,
  colframe=blues,
  coltitle=cvprblue,   % 改成蓝色标题文字
  fonttitle=\bfseries,
  boxrule=0.8pt,
  arc=2mm,
  left=8pt,
  right=8pt,
  top=6pt,
  bottom=6pt,
  title=#1
}

\newtcolorbox{innerbox}[2][]{%
  enhanced,
  breakable,
  colback=gray!5,
  colframe=#1!90!black,
  fonttitle=\bfseries,
  title=#2,
  boxrule=0.6pt,
  arc=2mm,
  left=5pt,
  right=5pt,
  top=4pt,
  bottom=4pt
}
\usepackage[textsize=tiny]{todonotes}

\icmltitlerunning{Plan Then Action: High-Level Planning Guidance Reinforcement Learning 
for LLM Reasoning}

\begin{document}

\twocolumn[
  \icmltitle{Plan Then Action: High-Level Planning Guidance Reinforcement Learning 
for LLM Reasoning}

  % It is OKAY to include author information, even for blind submissions: the
  % style file will automatically remove it for you unless you've provided
  % the [accepted] option to the icml2026 package.

  % List of affiliations: The first argument should be a (short) identifier you
  % will use later to specify author affiliations Academic affiliations
  % should list Department, University, City, Region, Country Industry
  % affiliations should list Company, City, Region, Country

  % You can specify symbols, otherwise they are numbered in order. Ideally, you
  % should not use this facility. Affiliations will be numbered in order of
  % appearance and this is the preferred way.
  \icmlsetsymbol{equal}{*}
 
  \begin{icmlauthorlist}
    \icmlauthor{Zhihao Dou}{equal,aaa}
    \icmlauthor{Qinjian Zhao}{equal,bbb}
    \icmlauthor{Zhongwei Wan}{equal,ccc}
    \icmlauthor{Dinggen Zhang}{bbb}
    \icmlauthor{Weida Wang}{ddd}
    \icmlauthor{Benteng Chen}{eee}
    \icmlauthor{Towsif Raiyan}{aaa}
    %\icmlauthor{}{sch}
    \icmlauthor{Qingtao Pan}{aaa}
    \icmlauthor{Yang Ouyang}{fff}
    \icmlauthor{Chaoda Song}{aaa}
    \icmlauthor{Zhiqiang Gao$^{\dagger}$}{bbb}
    \icmlauthor{Shufei Zhang$^{\dagger}$}{zzz}
    \icmlauthor{Sumon Biswas}{aaa}
    %\icmlauthor{}{sch}
    %\icmlauthor{}{sch}
  \end{icmlauthorlist}

  \icmlaffiliation{aaa}{Case Western Reserve University, Cleveland, OH, USA}
 \icmlaffiliation{bbb}{Kean University, Union, NJ, USA}
  \icmlaffiliation{ccc}{The Ohio State University, Columbus, OH, USA}
  \icmlaffiliation{ddd}{Fudan University, Shanghai, China}
  \icmlaffiliation{zzz}{Shanghai Artificial Intelligence Laboratory, Shanghai, China}
\icmlaffiliation{eee}{The University of Hong Kong, Hong Kong, China}
 \icmlaffiliation{fff}{North Carolina State University, Raleigh, NC, USA}

  \icmlcorrespondingauthor{Zhiqiang Gao}{zgao@wku.edu.cn}
  \icmlcorrespondingauthor{Shufei Zhang}{zhangshufei@pjlab.org.cn}

  % You may provide any keywords that you find helpful for describing your
  % paper; these are used to populate the "keywords" metadata in the PDF but
  % will not be shown in the document
  \icmlkeywords{Machine Learning, ICML}

  \vskip 0.3in
]

% this must go after the closing bracket ] following \twocolumn[ ...

% This command actually creates the footnote in the first column listing the
% affiliations and the copyright notice. The command takes one argument, which
% is text to display at the start of the footnote. The \icmlEqualContribution
% command is standard text for equal contribution. Remove it (just {}) if you
% do not need this facility.

% Use ONE of the following lines. DO NOT remove the command.
% If you have no special notice, KEEP empty braces:
%\printAffiliationsAndNotice{}  % no special notice (required even if empty)
% Or, if applicable, use the standard equal contribution text:
% \printAffiliationsAndNotice{\icmlEqualContribution}
% \printAffiliationsAndNotice{\icmlEqualContribution}
\printAffiliationsAndNotice{
\icmlEqualContribution,
$^{\dagger}$Corresponding authors.
}

\begin{abstract}
Large language models (LLMs) demonstrate strong reasoning abilities via Chain-of-Thought (CoT), but their token-level generation encourages local decisions and lacks global planning, often leading to redundant or inaccurate reasoning. Existing methods, such as tree-based search and reinforcement learning (RL), attempt to address this issue but incur high computational costs and still struggle to produce reliable reasoning trajectories.
% %%%%%%%
To address these challenges, we propose \textbf{P}lan-\textbf{T}hen-\textbf{A}ction Enhanced Reasoning with \textbf{G}roup \textbf{R}elative \textbf{P}olicy \textbf{O}ptimization (\textbf{\alg}), a two-stage framework designed to jointly improve high-level planning and fine-grained CoT reasoning. 
% 可缩减
Specifically, in the first stage, a given LLM is responsible for summarizing CoT reasoning into compact high-level guidance, which is then leveraged for supervised fine-tuning. Then, we introduce a guidance-aware reinforcement learning method that jointly optimizes the final output and the quality of guidance, enhancing reasoning effectiveness.
% %%%%
We evaluate \alg on ten reasoning benchmarks across mathematics and natural sciences, using five diverse base models spanning multiple data modalities. The results show that \alg consistently delivers significant improvements across models and tasks, demonstrating strong effectiveness and generalization.

\end{abstract}

\section{Introduction}

\begin{figure*}[t]
    \centering
    \begin{subfigure}[b]{0.30\linewidth}
        \centering
        \includegraphics[width=\linewidth]{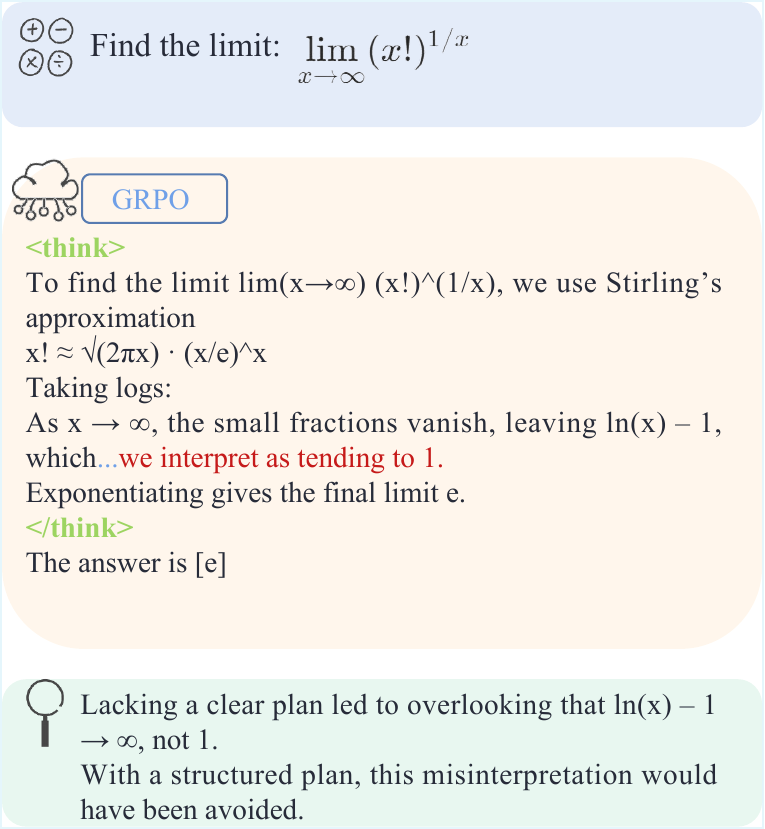}
        \caption{}
        \label{fig:grpo}
    \end{subfigure}
    \begin{subfigure}[b]{0.30\linewidth}
        \centering
        \includegraphics[width=\linewidth]{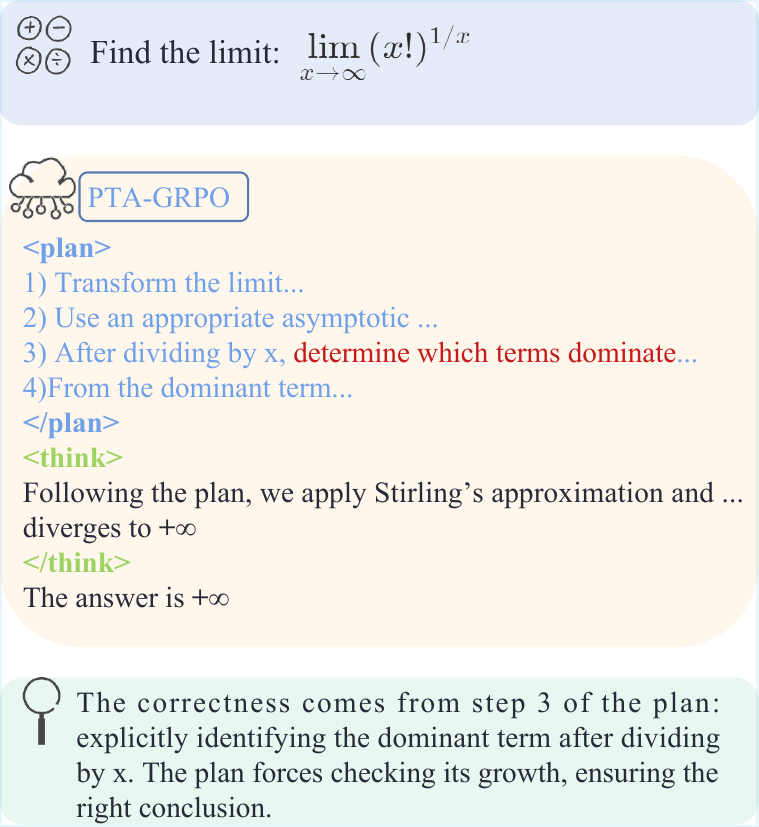}
        \caption{}
        \label{fig:PTA1}
    \end{subfigure}
    \begin{subfigure}[b]{0.30\linewidth}
        \centering
        \includegraphics[width=\linewidth]{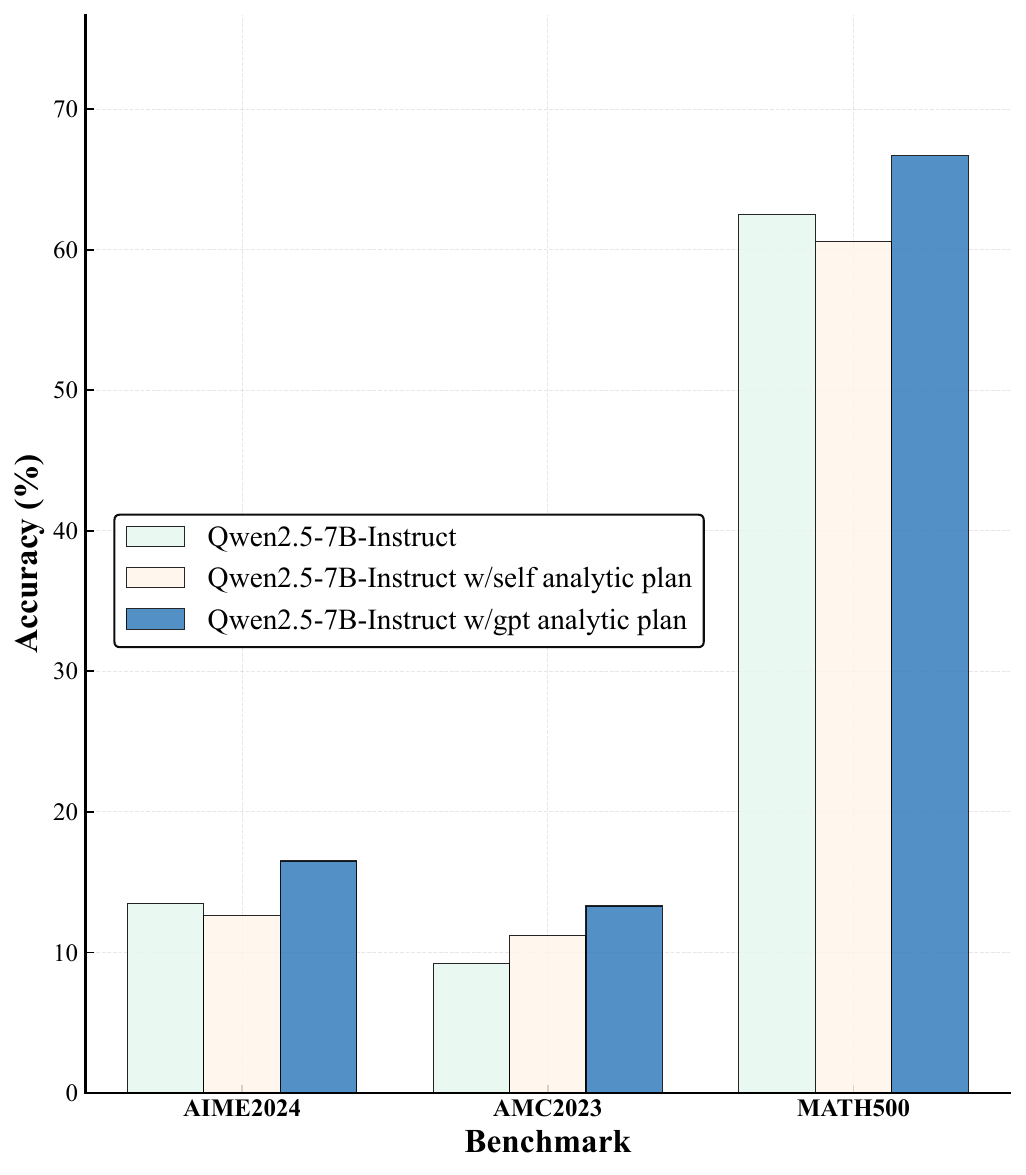}
        \caption{}
        \label{fig:PTA2}
    \end{subfigure}

    \caption{\small{(a) GRPO reasoning process. (b) \alg reasoning process. (c) Impact of analytic plan. In (c), the accuracy of different reasoning modes, where Qwen2.5-7B-Instruct is considered as the base model. 
    Green indicates the base model using CoT reasoning, yellow indicates the base model reasoning with its own self-generated analytic plan, and blue indicates the base model reasoning with an analytic plan generated by GPT-o1. More test cases of \alg are shown in Appendix~\ref{sec: test cases.}.}}
    \label{fig:impact}
\end{figure*}

Large Language Models (LLMs) have recently demonstrated remarkable reasoning abilities across a wide range of complex tasks \citep{xu2025towards, plaat2024reasoning, ke2025survey}, such as mathematics \citep{zhang2024llama,wu2023mathchat,liu2023improving} and programming \citep{jiang2024survey}, by leveraging Chain-of-Thought (CoT) reasoning \citep{wei2022chain}. Models with strong reasoning capabilities, including Qwen-3 \cite{yang2025qwen3}, DeepSeek-R1 \citep{wu2024deepseek}, Seed-1.5 thinking \citep{seed2025seed1}, and GPT-5 thinking \citep{openai_gpt5_systemcard}, adopt CoT as a central mechanism to structure their reasoning processes.  
However, CoT decoding in LLMs is still a token-level Markov Decision Process (MDP) \citep{ouyang2022training,wan2025srpo,liu2025understanding}: the output of each token is determined by the context sequence generated previously. Under this setting, mainstream decoding is both autoregressive (each decision conditions only on the prefix) and locally greedy (it optimizes short-horizon token likelihood, e.g., via greedy/low-temperature choices). This combination preserves local consistency but offers little global planning, often yielding redundant or drifting chains of thought and propagating early mistakes across long horizons \citep{yao2023tree,qu2025survey,wan2025srpo}.

Prior work augments LLM reasoning with tree-style algorithms \citep{zhang2024llama,yao2023tree,wang2024litesearch} such as Monte Carlo Tree Search \citep{zhang2024llama} or heuristic generation tree \citep{li2025treepo} to widen exploration beyond single-path decoding. While effective in some cases, these approaches hinge on repeated external queries to the LLM, incurring substantial time and compute \citep{wang2024litesearch}. Crucially, they do not strengthen the model’s internal reasoning: performance stems from outside search. When the model cannot verify intermediate steps, the search simply amplifies bad branches and collapses \citep{feng2023alphazero}.
In parallel, recent works inject reflection or backtracking behaviors via RL \citep{wan2025srpo,wang2025vl,gandhi2025cognitive}. Such behaviors can, in principle, re-route trajectories and escape local optima \citep{gandhi2025cognitive,liang2026render}. Yet when triggered on corrupted partial solutions, the model tends to reflect on its own errors, reinforcing them and drifting farther from the correct path. This occurs largely due to the absence of a global plan to guide self-reflection, leaving the model without a reliable mechanism to recover.
These limitations motivate a new paradigm that \textit{improves internal planning rather than relying on external search or post-hoc self-correction}.

Motivated by human problem-solving behavior, where an abstract plan is sketched before execution \citep{kahneman2011thinking}, we explore whether LLM reasoning can similarly benefit from an explicit plan-then-execute paradigm. Concretely, an LLM first generates a compact analytic plan that provides global guidance—such as subgoal decomposition and execution order—and then conditions detailed chain-of-thought (CoT) reasoning on this plan. This global conditioning mitigates local myopia and reduces redundant or inconsistent reasoning trajectories.
However, analytic planning is itself a challenging capability. We observe that weaker models, such as Qwen-2.5-7B-Instruct \citep{bai2023qwen}, often fail to produce reliable plans. As shown in Fig.~\ref{fig:PTA2}, low-quality plans can degrade downstream CoT and final answers, whereas plans generated by stronger models (e.g., GPT-o1) lead to consistent performance gains. This contrast indicates that reasoning improvements critically depend on the quality of analytic plans rather than the mere presence of planning. However, existing plan–action based approaches \cite{yao2022react,wang2023plan,wan2025rema,dou2025dsadf,liang2026vanim} emphasize structural separation between planning and execution, without explicitly formulating plan quality as a reward signal for policy optimization.

These observations suggest that analytic planning should be treated as a first-class optimization target. A natural approach is to leverage reinforcement learning (RL), which enables trajectory-level, non-differentiable optimization and can align plan generation with downstream reasoning behavior. However, existing outcome-based RL with verifiable rewards (RLVR) methods—such as GRPO \citep{shao2024deepseekmath} and DAPO \citep{yu2025dapo}—optimize only final answer correctness, ignoring the quality of intermediate planning and reasoning (Fig.~\ref{fig:overall}). Consequently, poorly structured plans and CoT trajectories may receive the same reward as well-organized ones, limiting the model’s ability to learn robust global planning.
Meanwhile, certain plan-based reasoning architectures \cite{yao2022react,zhou2022least,erdogan2025plan} explicitly decouple the planning and execution stages. Such a design makes it difficult to obtain reliable and verifiable reward signals for the planning component itself, as its quality cannot be directly or independently evaluated.

To address this limitation, we propose \textbf{\alg} (\textbf{p}lan-\textbf{t}hen-\textbf{a}ction enhanced reasoning with
\textbf{G}roup \textbf{R}elative \textbf{P}olicy \textbf{O}ptimization), a two-stage plan-then-reason training framework that explicitly optimizes both analytic planning and detailed reasoning. Unlike prior works \cite{yao2022react,wan2025rema} that treat the plan merely as an additional structural component, we formulate the plan as an explicit optimization variable. In the first stage, we introduce a Planning-Structured Reasoning cold-start strategy, where a given LLM summarizes ground-truth CoT into concise high-level guidance that captures core concepts and global reasoning structure. This guidance, paired with the original CoT, forms a supervised fine-tuning (SFT) dataset that initializes explicit planning ability—an ability largely absent from base pre-trained models \citep{gandhi2025cognitive,yue2025does,li2025treepo}.
In the second stage, we develop a plan-guidance-aware RL method based on GRPO. Unlike standard GRPO, which rewards only the final response, our approach incorporates a refined reward signal that evaluates the quality of the generated high-level guidance during reasoning. This design encourages the model to produce not only correct answers but also effective and precise analytic plans, leading to more stable and globally guided reasoning trajectories. 
Through rigorous theoretical derivations and empirical validation, we demonstrate that the proposed reinforcement learning method can significantly increase the mutual information between predicted answers and ground-truth answers, thereby improving the accuracy of the model’s generated outputs.
Experimental results demonstrate that our method achieves significant performance improvements across multiple benchmarks and domains, and is applicable to both LLMs and MLLMs.
Our main contributions are summarized as follows:
\begin{itemize}[leftmargin=*]
    \item \textbf{A novel two-stage plan-reasoning framework:} We propose \alg, a two-stage training framework, including high-level guidance planning and guidance-aware reinforcement learning, to foster explicit higher-order planning and reasoning abilities in LLMs.
    
    \item \textbf{High-level guidance as supervision signal:} In the supervised fine-tuning stage, we leverage a given LLM to summarize raw chain-of-thought (CoT) into concise high-level guidance, which is combined with the original CoT, providing stronger initialization for reasoning.
    
    \item \textbf{Impressive reasoning performance:} We evaluate the proposed method across five different models, ten benchmarks, and a variety of domains and modalities. Experimental results show that it achieves significant improvements in the vast majority of settings and attains excellent performance in most evaluated scenarios.
\end{itemize}

\section{Preliminaries and related work}

\subsection{Reasoning in Large Language Models}
\label{sub:reason}

The reasoning of an LLM can be formalized as a token-level Markov Decision Process (MDP) \citep{ouyang2022training,wan2025srpo,liu2025understanding}, where the state is the context sequence, the action is the next token, and the policy is the model’s conditional distribution. Given a question $q$, a response $\mathfrak{o}=[\mathfrak{o}^1,\dots,\mathfrak{o}^T]$ is sampled step by step from $\pi_{\theta}(\cdot \mid q,\mathfrak{o}^{<t})$. Current inference typically relies on CoT, producing a reasoning chain $c$ and final answer, but this purely autoregressive process lacks global planning, often leading to redundancy and incoherence \citep{wan2025srpo}.

\subsection{Group Relative Policy Optimization and Its Extensions}

GRPO \cite{shao2024deepseekmath}, proposed by DeepSeek, enhances LLM reasoning without value models by sampling multiple responses per prompt and using the group average reward as a baseline. This simple mechanism has proven effective in mathematical reasoning, code generation, and QA.
Subsequent variants refine GRPO from different perspectives: SRPO \cite{zhang2025srpo} reuses samples via history resampling; DAPO \cite{yu2025dapo} filters extreme cases with dynamic sampling; Dr.GRPO \citep{liu2025understanding} mitigates length bias; EMPO \cite{zhang2025right} optimizes semantic entropy directly; and SEED-GRPO \cite{chen2025seed} integrates entropy as an uncertainty measure for more conservative updates. While these methods substantially improve mathematical reasoning, they do not explicitly target higher-order reasoning abilities.
GiGPO \cite{feng2025group} introduces a hierarchical episode-level and step-level relative advantage estimation within the group-based reinforcement learning framework, enabling fine-grained credit assignment for multi-turn LLM agents while significantly improving long-horizon decision-making performance without requiring additional rollouts or critic networks.

\subsection{Motivation}

To address the lack of global guidance in LLM reasoning, which often leads to redundancy or off-topic reasoning, inspired by human thinking habits for complex tasks or problems \citep{kahneman2011thinking,kahneman2013prospect}, we introduce a concise high-level plan $t$ as an outline before generating the detailed CoT $c$ and its corresponding answer. Formally, the model’s output can be represented as $\mathfrak{o} = {t, c}$, where $t$ provides the overall problem-solving direction without involving concrete computational steps, and $c$ is then generated conditioned on both the question $q$ and the plan $t$, i.e., $c = \pi_\theta(\cdot \mid q, t)$. The CoT $c$ and its final answer are guided by the high-level plan $t$.
This \textit{plan-then-reason} mechanism equips the reasoning process with global guidance, leading to more concise and accurate CoT.

Therefore, in GRPO optimization (the formulas are shown in Appendix~\ref{sec:grpo}) in our study, the objective goes beyond simply ensuring the correctness of the answer in $\mathfrak{o}$. It also includes enhancing the quality of the high-level plan $t$, with the aim of producing $t$ more accurately and effectively. By improving $t$, the model receives structured guidance that can better direct the generation of the CoT $c$ and, consequently, the final answer. 
This dual focus rewards correct answers while strengthening high-quality intermediate reasoning, resulting in more robust and generalizable reasoning.

\begin{figure*}[t]
    \centering
    \includegraphics[width=0.84\linewidth]{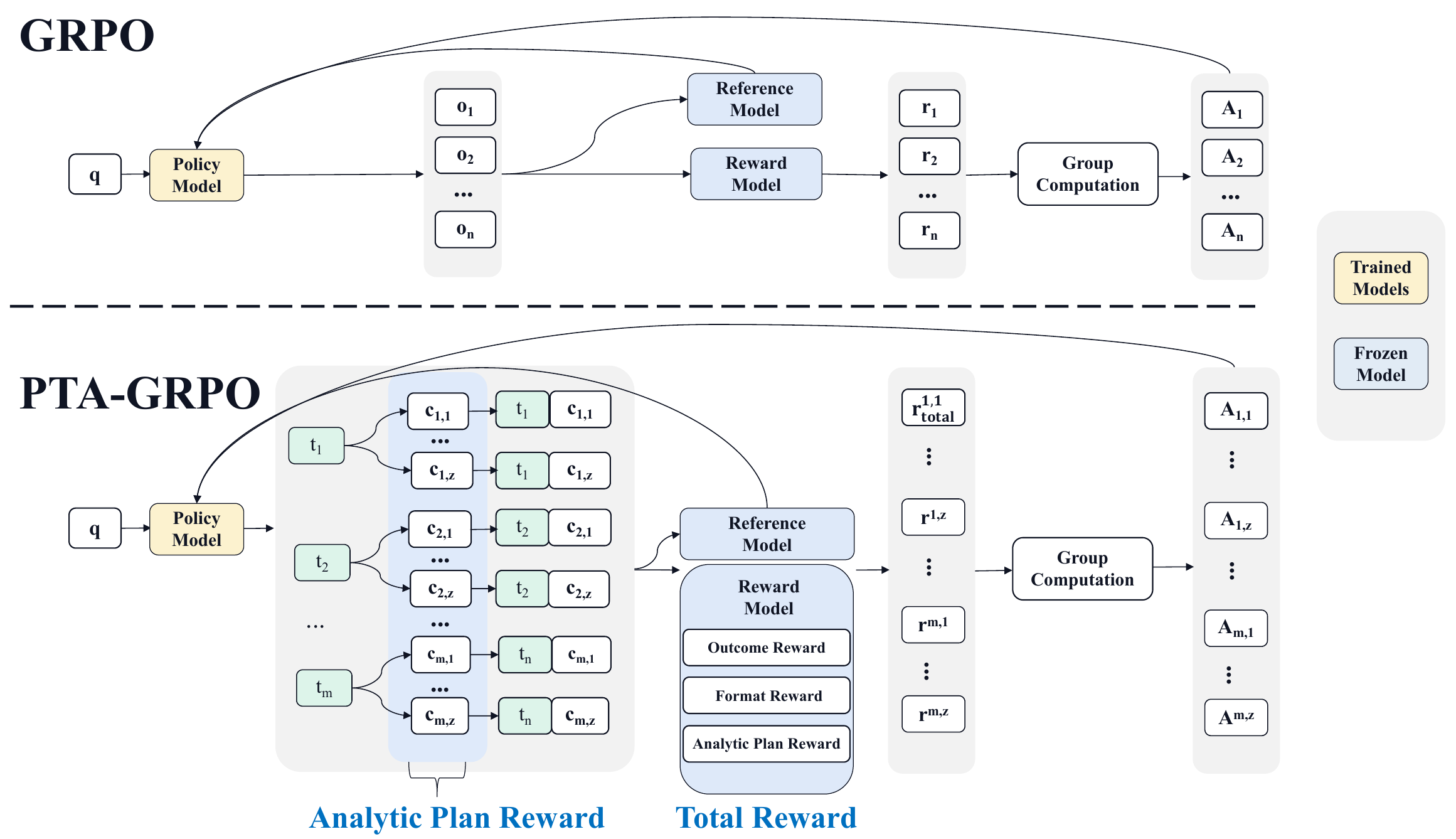}
    \caption{\small{Comparison between GRPO and \alg. It is worth noting that, to ensure a fair comparison, the number of rollout responses is kept the same between GRPO and PTA-GRPO.}}
    \label{fig:overall}
    
\end{figure*}

\section{Approach of \alg}
 
We introduce the \alg training framework, which consists of two key components. 
\textbf{(1) Planning Structured Reasoning Cold-Start (PSR-CS).} 
This module adopts an SFT-based cold-start strategy by explicitly introducing a general analytic plan before detailed reasoning, providing high-level guidance for subsequent reasoning and answer generation.
\textbf{(2) Planning Structure-Guided Reinforcement Learning (PSG-RL).}
We propose a GRPO-based structure-guided reinforcement learning framework that explicitly models the quality of analytic plan as a reward signal and incorporates it into the optimization objective, thereby improving the model’s structured reasoning capability and answer generation performance.

\subsection{Planning Structured Reasoning Cold-Start (PSR-CS)}
\label{sec:sft}

\myparatights{Analytical-Guided SFT Dataset Construction}
For LLMs, the ability to perform effective planning plays a critical role in problem solving. However, existing supervised fine-tuning (SFT) datasets typically include only chain-of-thought (CoT) reasoning and final answers, while overlooking the importance of high-level analytic plan prior to reasoning. To address this limitation, we propose an analytical-guided dataset consisting of three components: the problem, a general analytic plan, and the corresponding response containing CoT reasoning and the final answer.
Formally, the dataset is defined as $
D_{\text{PSR-CS}} = \{(q_i, t_i, c_i)\}_{i=1}^n,$
where $q_i$ denotes the problem, $t_i$ represents the general analytic plan, and $c_i$ contains the detailed reasoning process and final answer.
Unlike directly producing CoT reasoning~\citep{wei2022chain}, which often lacks global guidance, SFT explicitly injects the high-level problem-solving plan $t_i$ during training, enabling the model to learn the conditional distribution $c_i = \pi_\theta(\cdot \mid q_i, t_i)$ and to leverage global information when generating reasoning chains.
In our dataset, the analytic plan $t_i$ is enclosed within \texttt{<plan>...</plan>} tags, while the reasoning and answer in $c_i$ are structured using \texttt{<think>...</think>} and \texttt{<answer>...</answer>} tags, respectively, forming a hierarchical representation of planning, reasoning, and answering.
In contrast to prior approaches that require multi-turn interactions to obtain high-level guidance~\citep{yao2022react,zhou2022least}, our design integrates planning and reasoning--answering into a single compact response, allowing the model to complete planning and execution in one pass and enabling efficient joint optimization of plan--action components under RLVR in Section \ref{PTA}.

In practice, we sample 10K instances from the Openthoughts dataset~\citep{guha2025openthoughts} as the base data.
Then, a general analytic plan $t_i$ will be summarized from the given problem $q_i$ and its corresponding reasoning $c_i$ by employing an existing LLM (e.g., Qwen3-235B~\citep{yang2025qwen3}).
Notably, as the target of this process is to simply summarize the reasoning steps to a general plan, it will not heavily rely on the reasoning capability of the LLM.
To verify this nature, the experiments in Table \ref{tab:self} (Appendix \ref{sec:self}) demonstrate that the plan summarized by the given LLM itself also produces the obvious improvement.

\myparatights{SFT-based Cold-Start Initialization Optimization}
At this stage, we inject structured reasoning capabilities into the initial policy model $\pi_{\theta}$ through SFT.
Specifically, we optimize the model parameters by minimizing the discrepancy between the model outputs and the reference outputs in the analytical-guided dataset $D_{\text{SRCS}}$, enabling the model to gradually learn structured reasoning patterns. The fine-tuning process is formulated as:
\begin{equation}
\small
\theta_{\text{SFT}} = \underset{\theta}{\min}\quad 
\mathbb{E}_{(q_i,t_i,c_i) \in \mathcal{D}_{\text{SRCS}}}
\left[-
\sum_{i=1}^{n} \log\left( \pi_{\theta}(t_i, c_i \mid q_i) \right)
\right],
\end{equation}
where $\theta_{\text{SFT}}$ denotes the parameter set obtained through supervised fine-tuning, and $\pi_{\theta_{\text{SFT}}}$ represents the resulting policy model equipped with structured reasoning capabilities. 
Through SFT-based Cold-Start Initialization, we inject planning capabilities into the policy model in a direct manner, thereby expanding its knowledge and capacity for analytic planning and providing effective supervision signals \cite{shah2025rethinking} for accurate answer generation.

\subsection{Plan Structure-Guided Reinforcement Learning (PSG-RL)}
\label{PTA}

After obtaining the policy model $\pi_{\theta_{\text{SFT}}}$ from the SFT stage, the RL phase then focuses on improving the model's planning capability and ensuring its effective execution. At this stage, we not only consider the correctness of CoT $c$ and its answer as part of the reward signal, but also evaluate the quality of the analytic plan $t$, which is incorporated as another important aspect of the reward signal.

\subsubsection{Analytical Plan–Guided Reward Augmentation in GRPO}

In \alg, we design a composite reward function that integrates three aspects: the analytic plan reward ($r_{\text{analytic}}$) to encourage structured reasoning plans, the outcome accuracy reward ($r_{\text{outcome}}$) to ensure correct final results, and the structured format reward ($r_{\text{format}}$) to enforce clear and consistent output. Together, these rewards are combined into the final total reward $R_{\text{total}}$.

\myparatights{Analytical Plan Reward}
Since directly evaluating the quality of an analytic plan $t$ is infeasible, we adopt a computable surrogate objective that measures how likely a plan can guide a CoT reasoning process toward the correct answer. We define the analytic plan reward $r_{\text{analytic}}$ as this probability, which serves as a proxy for plan quality.

Given a question $q$, we construct a response group $G$ via a two-stage sampling process. First, the policy model samples $m$ candidate analytic plans $\{t_i\}_{i=1}^m$, where $t_i \sim \pi_{\theta}(\cdot \mid q)$ and each $t_i$ is a concise text-based outline for solving $q$. Then, following \citep{lu2025arpo}, for each plan $t_i$ we resample $z$ detailed CoT trajectories $\{c_{i,k}\}_{k=1}^z$ under its guidance, where $c_{i,k} \sim \pi_{\theta}(\cdot \mid t_i, q)$.
The resulting response group is
\[
G = \Bigl\{ \{(t_i, c_{i,k})\}_{k=1}^z \Bigr\}_{i=1}^m.
\]

Each element in $G$ forms a plan--CoT pair. The reward assigned to an analytic plan $t_i$ is defined as the empirical accuracy of its resampled outcomes:
\begin{equation}
\begin{aligned}
s_i &= \frac{1}{z}\sum_{k=1}^{z}\mathbb{I}[\hat{y}_{i,k}=y], &
r_{\mathrm{analytic}}(t_i)
=
\left[\operatorname{softmax}(\mathbf{s}/\tau)\right]_i,
\end{aligned}
\label{eq:analytic_reward}
\end{equation}
where $\mathbb{I}[\cdot]$ is the indicator function, $\hat{y}_{i,k}$ denotes the final predicted answer extracted from $c_{i,k}$, and $y$ is the ground-truth answer for $q$. The Softmax operation amplifies relative score differences, highlighting high-quality plans while suppressing low-quality ones.
%
% In contrast to prior hierarchical frameworks \citep{dou2025dsadf,liang2026render} that mainly separate planning from execution, \alg treats the abstract plan as an optimizable intermediate variable and converts its effectiveness into an indirect supervision signal for reinforcement learning.
%
Importantly, the CoT set $\{c_{i,k}\}_{k=1}^z$, used to assess analytic plan $t_i$ quality, is also incorporated into the policy update in Group $G$, enabling a dual-use mechanism without extra sampling overhead. Compared to RLVR methods with the same rollout budget (e.g., GRPO), our approach does not incur noticeable additional time cost. RL training time is detailed in Appendix~\ref{sec:time}.

Optimizing the policy with $r_{\text{analytic}}(\cdot)$ encourages the generation of more accurate analytic plans and increases the probability of correct prediction $\Pr(\hat{y} = y \mid t, q)$.
In contrast to traditional RLVR methods \citep{yu2025dapo,feng2025group}, which rely solely on outcome-based supervision and cannot directly supervise intermediate reasoning, the analytic plan reward $r_{\text{analytic}}$ enables direct assessment of intermediate reasoning trajectories and assigns higher rewards to those more likely to succeed.
Section~\ref{sec:theory} shows, both theoretically and empirically, that optimizing $r_{\text{analytic}}$ increases the mutual information between $y$ and $\hat{y}$, thereby enhancing reasoning ability.

\myparatights{Outcome Reward}
The outcome reward, defined as $r_{\text{outcome}}$, is a result-based terminal reward similar to GRPO, used to evaluate whether the predicted answer aligns with the ground truth. For each plan–CoT response $(t_i, c_{i,k})$, the outcome reward $r_{\text{outcome}}$ is defined as follows:
\begin{equation}
 r_{\text{outcome}} = \begin{cases} 1,&\hat{y}_{i,k} = y ,\\ 
 0,&\text{else}.
 \end{cases}\label{eq:i_eff}
\end{equation}
The outcome reward $r_{\text{outcome}}$ encourages the policy model to learn to follow the analytic plan $t_i$ and to develop the ability to generate correct answers that strive for correctness.

\myparatights{Format Reward}
To ensure structural consistency and concise outputs, we incorporate a format reward $R_{\text{format}}$ that enforces a predefined response template and discourages overly long responses. The detailed formulation is in the Appendix~\ref{sec:format}.

\myparatights{Total Reward} The above three rewards together constitute the total reward $R_{\text{total}}$ for each response as:
\begin{equation}
    R_{\text{total}} = R_{\text{outcome}} + \beta \cdot R_{\text{analytic}} +R_{\text{format}},
\end{equation}
where $\beta$ represents the hyperparameter.
We first obtain a total reward set $\{\{r_{total}^{i,k}\}_{i=1}^m\}_{k=1}^z$, where $r_{total}^{i,k}$ denotes the total reward of the $k$-th CoT generated under the guidance of the $i$-th analytic. Based on this reward, we compute the corresponding advantage function $A_{i,k}$ using Eq.~\ref{advantages}, and subsequently incorporate it into the update rule in Eq.~\ref{grpo:fun} to optimize the model. 

\begin{table}[h]
\centering
\caption{\small{Performance comparison of different post-training methods using various base models. \textbf{Bold} is best per block.}}
\scalebox{0.70}{ % 调整缩放比例，比如 0.9=90% 大小
\begin{tabular}{c|ccccc}
\midrule[1.2pt]
Method                       & MATH500        & AIME24         & AIME25         & AMC23          & Average   \\ \midrule
\textbf{Qwen2.5-7B-Instruct} & 63.93          & 14.25          & 3.77           & 54.95          & 34.23 \\
GRPO                         & 90.67          & 30.15          & 25.20          & 73.21          & 54.81 \\
DAPO                         & 91.87          & 31.64          & 22.53          & 74.12          & 55.04 \\
CPL                          & 87.53          & 28.55          & 24.49          & 72.53          & 53.28 \\
Full-Step-DPO                & 89.25          & 27.85          & 23.39          & 69.52          & 52.50 \\
ORZ                          & 88.55          & 30.44          & 26.27          & 72.93          & 54.55 \\
\rowcolor{blues}
\alg                          & \textbf{92.53} & \textbf{34.53} & \textbf{29.25} & \textbf{77.51} & \textbf{58.46} \\ \midrule

\textbf{LLaMA3.2-3B}         & 36.64          & 3.97           & 2.88           & 19.68          & 15.79 \\
GRPO                         & 62.23          & 19.55          & 18.29          & 46.44          & 36.63 \\
DAPO                         & 63.97          & 18.97          & 18.55          & 46.44          & 36.98 \\
\rowcolor{blues}
\alg                          & \textbf{66.25} & \textbf{24.92} & \textbf{20.75} & \textbf{49.55} & \textbf{40.37} \\ \midrule

\textbf{Qwen3-8B}            & 90.65          & 65.27          & 52.33          & 88.58          & 74.21 \\
GRPO                         & \textbf{94.33} & 71.57          & 54.94          & 93.51          & 78.59 \\
DAPO                         & 93.09          & 69.17          & 52.77          & 92.35          & 76.85 \\
CPL                          & 92.22          & 70.79          & 52.33          & 91.52          & 76.72 \\
Full-Step-DPO                & 93.03          & 71.78          & 50.65          & 92.37          & 76.96 \\
ORZ                          & 93.95          & 68.67          & 53.97          & 92.27          & 77.22 \\
\rowcolor{blues}
\alg                          & 93.79          & \textbf{72.68} & \textbf{56.14} & \textbf{93.77} & \textbf{79.10} \\ \midrule

\textbf{Qwen3-14B}           & 91.75          & 72.39          & 70.55          & 93.97          & 82.17 \\
GRPO                         & 92.48          & 73.62          & 71.33          & 94.87          & 83.08 \\
DAPO                         & 93.27          & 72.55          & 71.77          & \textbf{95.66} & 83.31 \\
\rowcolor{blues}
\alg                          & \textbf{96.15} & \textbf{75.44} & \textbf{73.29} & 95.17          & \textbf{85.01} \\ \midrule[1.2pt]
\end{tabular}
}
 
\label{tab:main}
\end{table}

\myparatights{Advantages of \alg}
Compared with standard GRPO, which primarily relies on sparse task-level accuracy supervision, our guidance-aware \alg framework introduces several critical improvements. \textbf{First}, powered by the analytic-plan reward~$r_{\text{analytic}}$, the model gains the ability to \emph{evaluate} its intermediate reasoning process, which RLVR cannot achieve with purely outcome-based signals. This mechanism drives the model to construct higher-level analytic plans and use them to guide more reliable CoT reasoning. \textbf{Second}, the outcome reward $r_{\text{outcome}}$ encourages the policy model to follow the analytic plan and enhance its reasoning capability under such structured guidance. Besides, following \citep{gandhi2025cognitive}, PTA-GRPO enhances LLM self-reflection in reinforcement learning by adjusting the prompt in Appendix \ref{sec: test cases.}, allowing the model to correct later steps even when the initial plan is flawed.

\subsection{Theoretical Performance Analysis}
\label{sec:theory}
 
In this section, we theoretically analyze the impact of optimizing $r_{\text{analytic}}$ on the error probability of the policy model. Our theoretical findings are as follows.

\begin{theorem}
\label{thm:1}
Let $q$ denote the input question, $t$ the analytic plan, $\hat{y}$ the answer predicted by the policy model, and $y$ the ground-truth answer. With error probability $p_{\text{error}}$, it holds that:
 % Adjust only the formula size
{\small
\begin{align*}
    p_{\text{error}} \leq \frac{1}{2} \big[ H(y) - I(\hat{y},y \mid t,q) \big], 
    \quad p_{\text{error}} = \Pr(y \neq \hat{y}),
\end{align*}
}
where $H(\cdot)$ denotes the entropy, $I(\cdot)$ denotes the mutual information.

\end{theorem}

The proof can be seen in appendix \ref{sec:proof}. Leveraging the conclusion from \citep{qian2025demystifying}, since $H(y)$ depends solely on the fixed distribution of the answer and is independent of the model's reasoning steps, it can therefore be regarded as a constant.
In our Theorem~\ref{thm:1}, the upper bound of the error probability 
$p_{\text{error}}$ is governed by the conditional mutual information
$
I(\hat{y}; y \mid t, q),
$
which measures the statistical dependence between the predicted output 
$\hat{y}$ and the true label $y$, given the auxiliary analytic plan $t$. In other words, the larger the shared information 
between $\hat{y}$ and $y$ under the guidance of the analytic plan $t$, the tighter the upper achievable limit on the probability of error. \

\begin{proposition}
\label{pro}
Let $t_1$ and $t_2$ be any analytic plans, and let $r_{\mathrm{analytic}}(t_1)$ and $r_{\mathrm{analytic}}(t_2)$ denote their corresponding analytic rewards. Let $\hat{y}_1$ and $\hat{y}_2$ be the answers induced by executing $t_1$ and $t_2$, respectively. If $
r_{\mathrm{analytic}}(t_1) \ge r_{\mathrm{analytic}}(t_2)$,
it holds that
\begin{equation}
H(y| \hat{y}_1,t_1,q) \leq H(y| \hat{y}_2,t_2,q).
\end{equation}
\end{proposition}

% \myparatights{Remark}
\begin{remark}
By the definition of mutual information,  
$
I(\hat{y};y \mid q,t) = H(y \mid q,t) - H(y \mid \hat{y}, q,t).
$ 
Note that $H(y \mid q,t)$ is solely determined by the underlying data distribution of $(q,t,y)$ and is independent of the model’s prediction $\hat{y}$. Hence, $H(y \mid q,t)$ can be regarded as a constant with respect to the learning or inference process. 
As shown in Proposition~\ref{pro}, as the analytic plan reward $r_{\text{analytic}}$ increases, the conditional entropy $H(y \mid \hat{y}, q, t)$ decreases, which in turn implies a larger mutual information $I(y;\hat{y} \mid q,t)$ between the prediction $\hat{y}$ and the ground-truth $y$. In particular, we have $\max_t r_{\text{analytic}}(t) \Longleftrightarrow \max_t I(y;\hat{y} \mid q,t)$. Therefore, optimizing the analytic plan reward $r_{\text{analytic}}$ effectively improves the model’s reasoning ability.

\end{remark}

\myparatights{Empirical Analysis}To assess the reliability of the above theory, we further examine the relationship between mutual information and the plan reward. Following \citep{qian2025demystifying}, we use the Hilbert–Schmidt Independence Criterion (HSIC) to estimate the mutual information between the predicted answer $\hat{y}$ and the ground-truth answer $y$. As shown in Fig.~\ref{fig:mutual}, the plan reward and the mutual information exhibit similar increasing trends, which is consistent with our theoretical claim
$\max_t r_{\text{analytic}}(t) \Longleftrightarrow \max_t I(y;\hat{y} \mid q,t)$
and thus supports its validity. 

% \begin{figure}
%     \centering
%     \includegraphics[width=0.8\linewidth]{figures/easyR1_math_Q2.5_plot.pdf}
%     \caption{Trend plots of mutual information and plan reward, where Qwen2.5-7B-Instruct is considered.}
%     \label{fig:mutual_7b}
% \end{figure}

\begin{figure*}[htbp]
    \centering
    % 三个子图并排
     \begin{subfigure}[b]{0.33\textwidth}
        \centering
        \includegraphics[width=\linewidth]{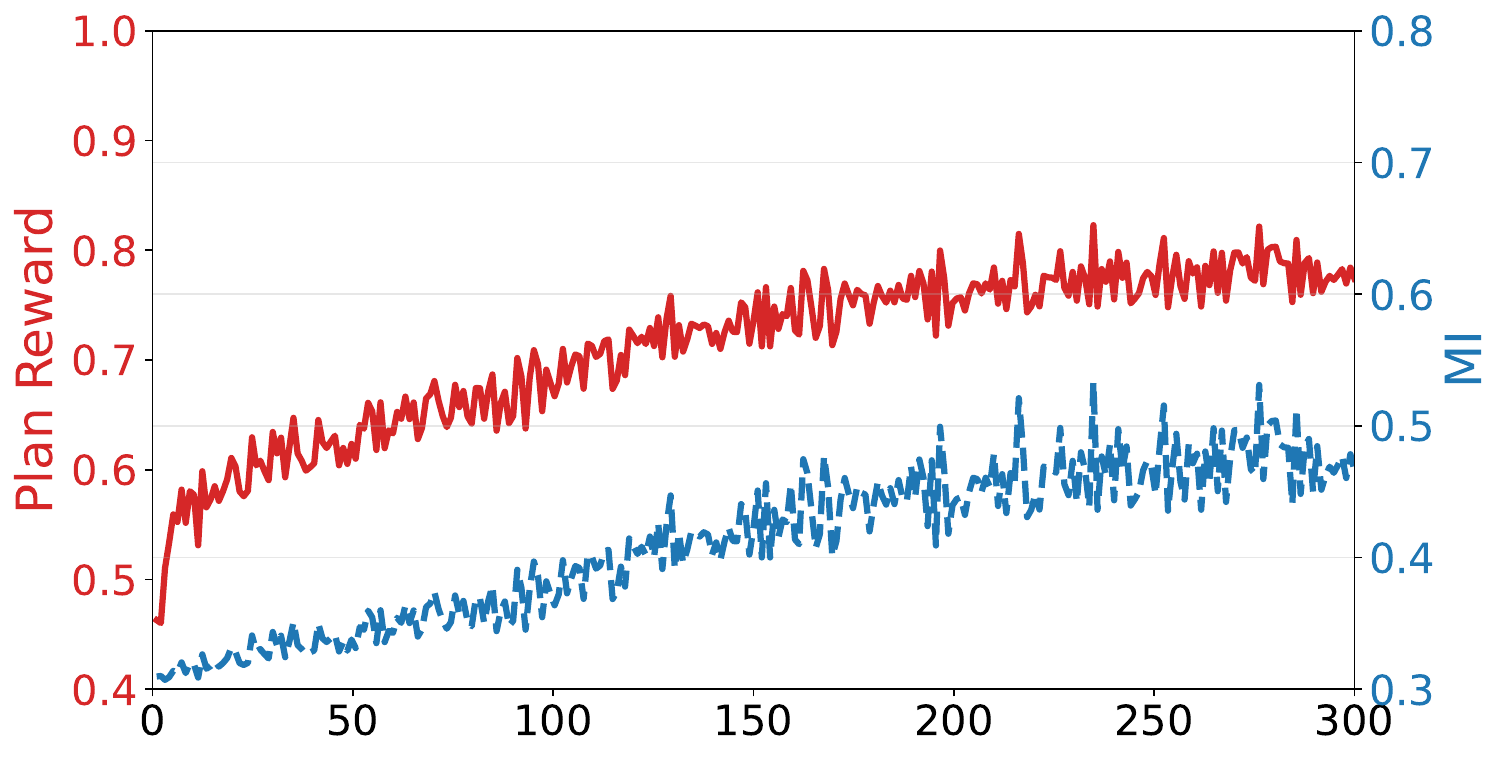}
        \caption{Qwen2.5-7B-Instruct}
    \end{subfigure}
    \hfill
    \begin{subfigure}[b]{0.33\textwidth}
        \centering
        \includegraphics[width=\linewidth]{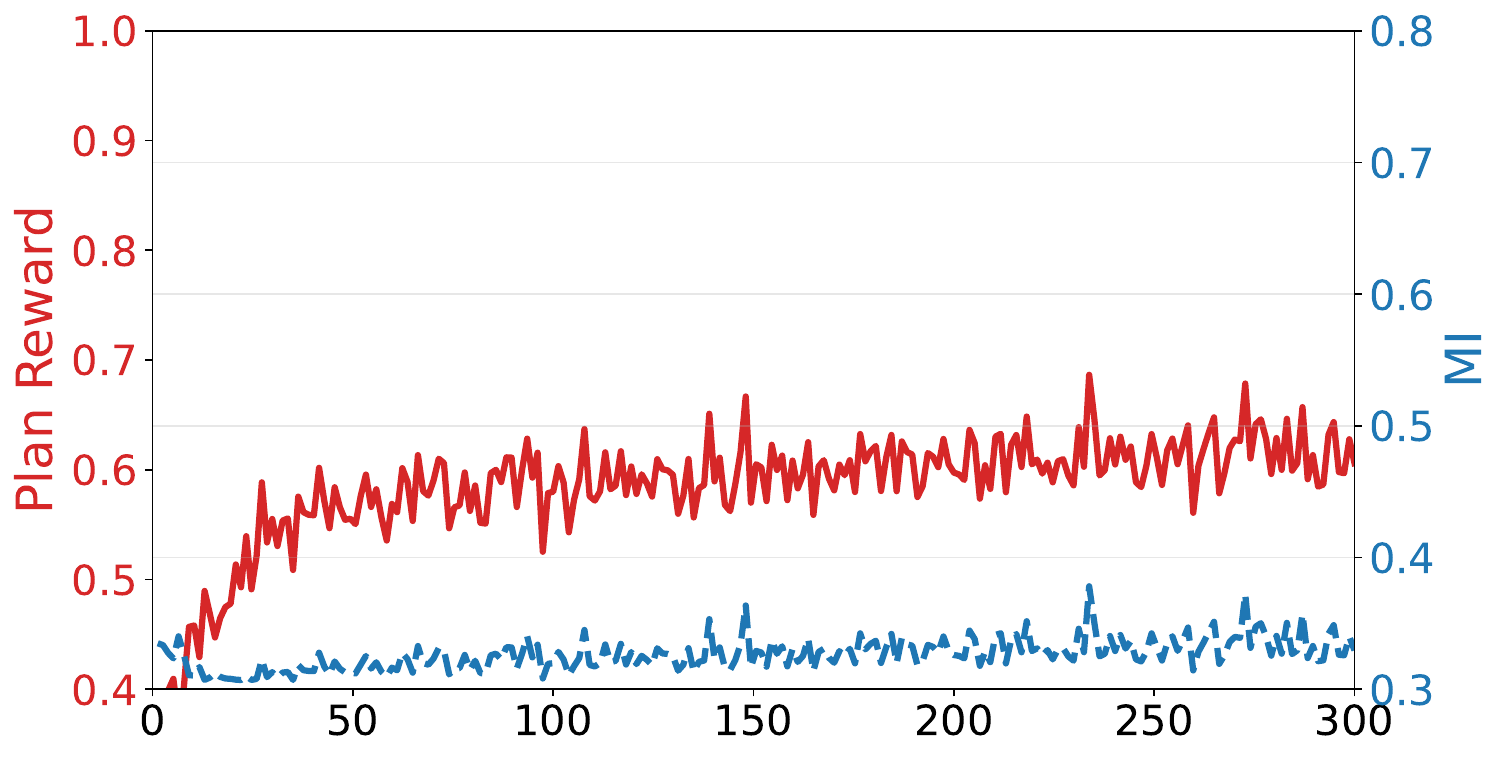}
        \caption{LLaMA3.2-3B}
    \end{subfigure}
    \hfill
    \begin{subfigure}[b]{0.33\textwidth}
        \centering
        \includegraphics[width=\linewidth]{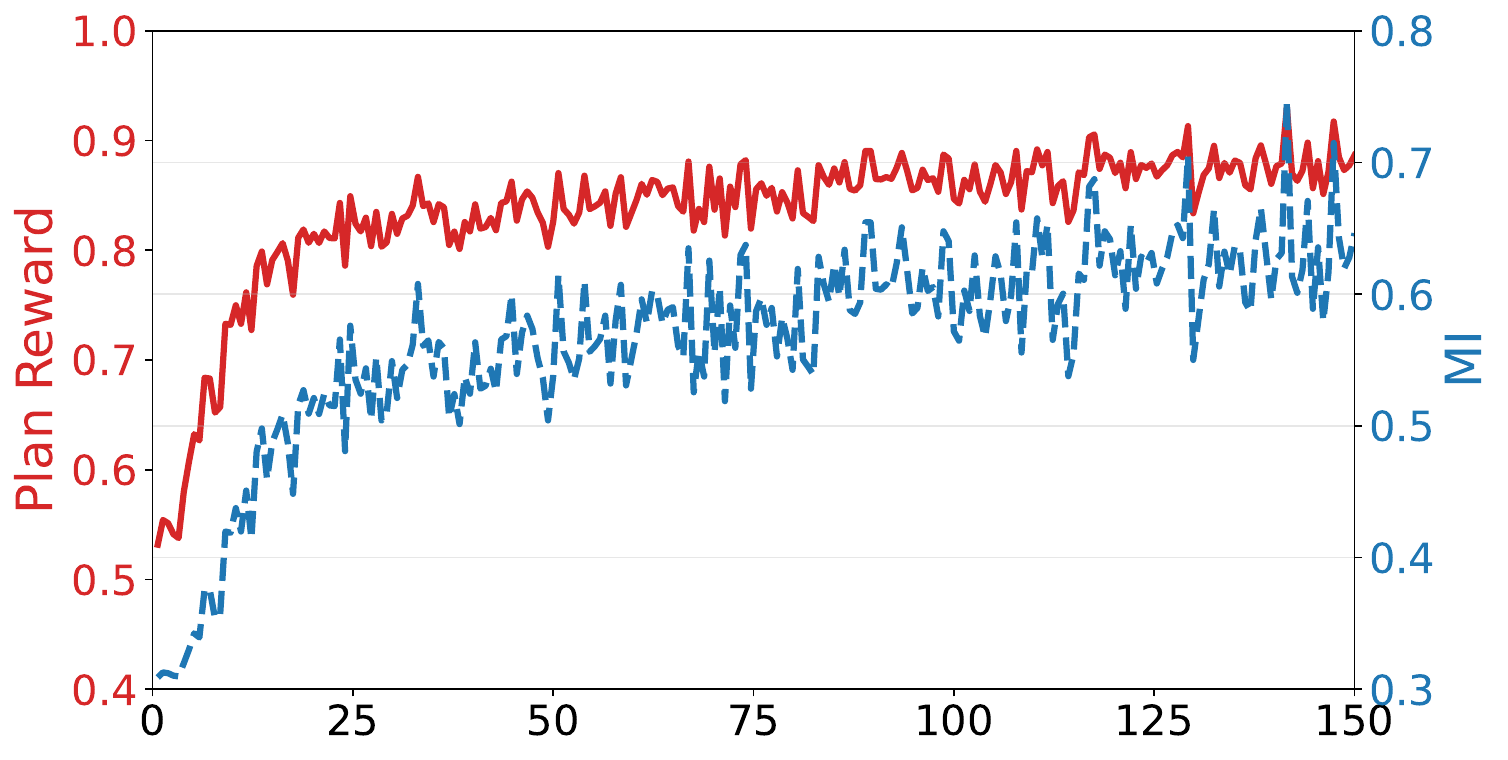}
        \caption{Qwen3-8B}
    \end{subfigure}

    \caption{\small{Trend plots of mutual information and plan reward across different models.}}
    \label{fig:mutual}
\end{figure*}

\begin{table}[h]
\centering
\footnotesize
\caption{\small{Impact of data scale of RL on \alg, where Qwen2.5-7B-Instruct is considered as base model. \textbf{Bold} is best per block.}}
\scalebox{0.8}{ % 比例可改成 0.85, 0.95 等
\begin{tabular}{c|ccccc}
\midrule[1.2pt]
Data scale & MATH500        & AIME24         & AIME25         & AMC23          & \multicolumn{1}{l}{Average} \\ \midrule
4k         & 82.27          & 27.22          & 21.03          & 65.22          & 48.94                      \\
8k         & 83.59          & 28.23          & 22.29          & 68.29          & 50.60                      \\
11k        & 84.23          & 29.33          & 24.51          & 69.37          & 51.86                      \\
14k        & 85.57 & 30.26 & 25.97 & 70.24 & 53.01
\\
30k &88.45 & 31.45 & 26.29 & 75.29 & 55.37
\\
\rowcolor{blues}
60K & \textbf{94.53} & \textbf{34.53} & \textbf{29.25} & \textbf{77.51} & \textbf{58.46}

\\ \midrule[1.2pt]
\end{tabular}
}
 
\label{tab:data_scale}
\end{table}
\section{Experiment}

\myparatights{Base Models}
To evaluate \alg, we adopt four base models of varying scales and series: LLaMA3.2-3B \citep{dubey2024llama}, Qwen2.5-7B-Instruct \citep{bai2025qwen2}, Qwen3-8B, and Qwen3-14B \citep{yang2025qwen3}, enabling a comprehensive assessment of its robustness across architectures. Training details are in Section~\ref{sec:setup}. For our vision-language model (VLM), we use Qwen2.5-7B-VL \citep{bai2025qwen2} as the base model to extend our experiments.

\myparatights{Training Datasets and Benchmarks}
For SFT, we use 10K samples from Openthoughts \citep{guha2025openthoughts} with injected planning knowledge (Section~\ref{sec:sft}). For RL, we sample 60K problems from DeeMath \citep{he2025deepmath}, which offers graded difficulty and is rigorously decontaminated to avoid benchmark leakage.
We evaluate \alg on AIME24, AIME25, MATH500, and AMC23, and report the average accuracy over 64 independent runs. Besides, we assess it on the general-purpose multimodal datasets MMMU-Pro \citep{yue2025mmmu}, MMMU \citep{yue2024mmmu}, and EMMA \citep{standley2023extensible}, and the scientific benchmark dataset MMK-12 \citep{meng2025mm}.

\begin{table}[H]
\footnotesize
\centering
\caption{\small{Ablation analysis on \alg, where Qwen2.5-7B-Instruct is considered as base model. \textbf{Bold} is best per block.}}
\scalebox{0.70}{
\begin{tabular}{c|ccccc}
\midrule[1.2pt]
Method                                          & MATH500        & AIME24         & AIME25         & AMC23          & \multicolumn{1}{l}{Average} \\ \midrule
\alg$_{\text{w/o SFT}}$          & 84.27          & 20.53          & 18.39          & 62.13          &  46.33                     \\
\alg$_{\text{w/o $ r_{\text{format}}$}}$   & 91.19          & 33.75 & 28.74          & 75.58          &  57.32                      \\
\alg$_{\text{w/o  $r_{\text{analytic}}$}}$ & 88.76          & 29.48          & 25.53          & 74.47         & 54.56                       \\
\rowcolor{blues}
\alg                             & \textbf{92.53} & \textbf{34.53}          & \textbf{29.25} & \textbf{77.51} & \textbf{58.46}              \\ \midrule[1.2pt]
\end{tabular}
}

\label{tab:abaltion}
\end{table}

\myparatights{Baseline}
We compare \alg with the base model and several RLVR like GRPO \citep{shao2024deepseekmath}, and DAPO \citep{yu2025dapo}. In addition, we compare our method with several advanced reinforcement learning algorithms, including CPL \citep{wang2024cpl}, Full-Step DPO \citep{xu2025full}, and ORZ \citep{hu2025open}. For fairness, all methods use the same SFT and RL data (differing only in the improved SFT portion). For a fair comparison, we use the same number of sampled responses as all selected RLVR methods, so their time consumption is nearly the same.

% \myparatights{Parameter Setup}

% \subsection{Experiment}

\subsection{Performance of \alg}

Table~\ref{tab:main} shows that our method (\alg) consistently outperforms both the base models and other RLVR approaches across different model scales and evaluation benchmarks. For relatively weaker backbones such as Qwen2.5-7B-Instruct and LLaMA3.2-3B, \alg delivers the most significant improvements, raising the average scores by over 20 points compared to the raw models and further surpassing GRPO and DAPO by clear margins. \alg demonstrates strong effectiveness. Detailed analysis is provided in Appendix~\ref{sec:sta_ana}.

\begin{table}[h]
\tiny
\caption{\small{The impact of datasets containing analytic planning on SFT. \textbf{Bold} is best per block.}}
\makebox[\linewidth]{ % 保证缩放后的表格居中
\scalebox{0.8}{
\begin{tabular}{c|c|ccccc}
\midrule[1.2pt]
Base Model                           & Method                              & MATH500        & AIME24         & AIME25         & AMC23          & Average \\ \midrule
\multirow{2}{*}{Qwen2.5-7B-Instruct} & SFT w/o planning & 78.28          & 21.66          & 19.66          & 60.53          & 45.03   \\
                                     & \cellcolor{blues} SFT w/ planning  & \cellcolor{blues}\textbf{80.40} & \cellcolor{blues}\textbf{25.25} & \cellcolor{blues}\textbf{20.33} & \cellcolor{blues}\textbf{63.75} & \cellcolor{blues}\textbf{47.43}   \\ 
                                     \midrule
\multirow{2}{*}{Qwen3-8B}            &  SFT w/o planning & 91.02          & 70.03          & 50.25          & 92.39          & 75.92   \\ 
                                     & \cellcolor{blues} SFT w/ planning  & \cellcolor{blues}\textbf{92.53} & \cellcolor{blues}\textbf{71.97} & \cellcolor{blues}\textbf{51.77} & \cellcolor{blues}\textbf{93.55} & \cellcolor{blues}\textbf{77.46}   \\ \midrule[1.2pt]
\end{tabular}
}}
\label{tab:sft}
\end{table}

\subsection{Impact of RL Data Scaling}

Table \ref{tab:data_scale} shows that, with Qwen2.5-7B-Instruct as the base model, scaling the RL training data from 4k to 60k leads to consistent improvements for \alg across four math benchmarks (MATH500 / AIME24 / AIME25 / AMC23), indicating that more RL data yields stable performance gains.
Notably, the improvement is not strictly linear: the gains are stronger at some stages and smaller at others, while a larger-scale setting still brings a clear boost. This suggests that scaling up the RL data continues to unlock additional capability rather than quickly saturating.

\subsection{Ablation Analysis}

Table~\ref{tab:abaltion} presents an ablation study of \alg on Qwen2.5-7B-Instruct. The full \alg achieves the best performance across all benchmarks (58.46 average), while removing SFT causes a substantial drop (46.33). Disabling either $r_{\text{format}}$ or $r_{\text{analytic}}$ also consistently hurts performance, with a larger degradation observed when removing $r_{\text{analytic}}$, indicating that both rewards contribute to final gains.

\begin{table}[h]
\tiny
\centering
\caption{Comparison between PTA-GRPO and other approaches on General-Benchmark and Science Benchmark, using Qwen2.5-7B-VL as the base model.}
\begin{tabular}{c|cccccc}
\midrule[1.2pt]
Method        & MMMU-Pro      & MMMU          & EMMA & Phys          & Chem          & Bio           \\ \hline
\textbf{Base} & 36.9          & 54.3          & 21.5 & 45.4          & 56.4          & 54.0          \\
SRPO          & 42.3          & 57.1          & 29.6 & 56.2          & 65.2          & 65.2          \\
\cellcolor{blues}PTA-GRPO      & \cellcolor{blues}\textbf{44.7} & \cellcolor{blues}\textbf{59.0} & \cellcolor{blues}\textbf{31.9} & \cellcolor{blues}\textbf{58.5} & \cellcolor{blues}\textbf{68.7} & \cellcolor{blues}\textbf{66.8} \\ \midrule[1.2pt]
\end{tabular}
\label{narrow}
\end{table}

\subsection{Impact of Analytic Plan on SFT}

Table \ref{tab:sft} compares standard SFT (w/o planning) with SFT on $\mathcal{D}_{\text{SRCS}}$ augmented by analytic plans (w/ planning). Across both base models, SFT with analytic planning consistently outperforms SFT without planning on all four benchmarks: for Qwen2.5-7B-Instruct, the average score rises from 45.03 to 47.43 (+2.40), and for Qwen3-8B, from 75.92 to 77.46 (+1.54), with gains appearing simultaneously on MATH500, AIME24/25, and AMC23 rather than as isolated fluctuations. This suggests that plans provide an important supervision signal, teaching the model transferable solution organization and reasoning trajectories that yield robust improvements across tasks and models.

\subsection{Empirical Evaluation on Generalization}

Beyond mathematics, we also evaluate on multimodal (Table \ref{narrow}), general-domain, and scientific benchmarks, including MMMU-Pro \citep{yue2025mmmu}, MMMU \citep{yue2024mmmu}, EMMA \citep{standley2023extensible}, and MMK-12 \citep{meng2025mm}. Using SRPO \citep{wan2025srpo} as baselines and following SRPO’s SFT/RL data for cold-start and training, PTA-GRPO with Qwen2.5-7B-VL consistently outperforms the Base model and SRPO on MMMU-Pro, MMMU, EMMA, and Phys/Chem/Bio benchmarks, achieving uniformly better metrics and stronger generalization reasoning.

\subsection{Results of test-time scaling}
We next examine the effectiveness of \alg under multiple sampling at test time. As shown in Fig. \ref{fig:topk}, \alg consistently outperforms GRPO on the AIME2025 dataset across Pass@1, Pass@4, Pass@8, and Pass@16. This demonstrates that \alg maintains high precision under low-sample conditions, while further exhibiting stronger solution coverage as the number of samples increases.
\begin{figure}[H] % r=右侧, l=左侧；后面是包围框宽度                       % 视情况微调顶端空白
  \centering
  \includegraphics[width=0.60\linewidth]{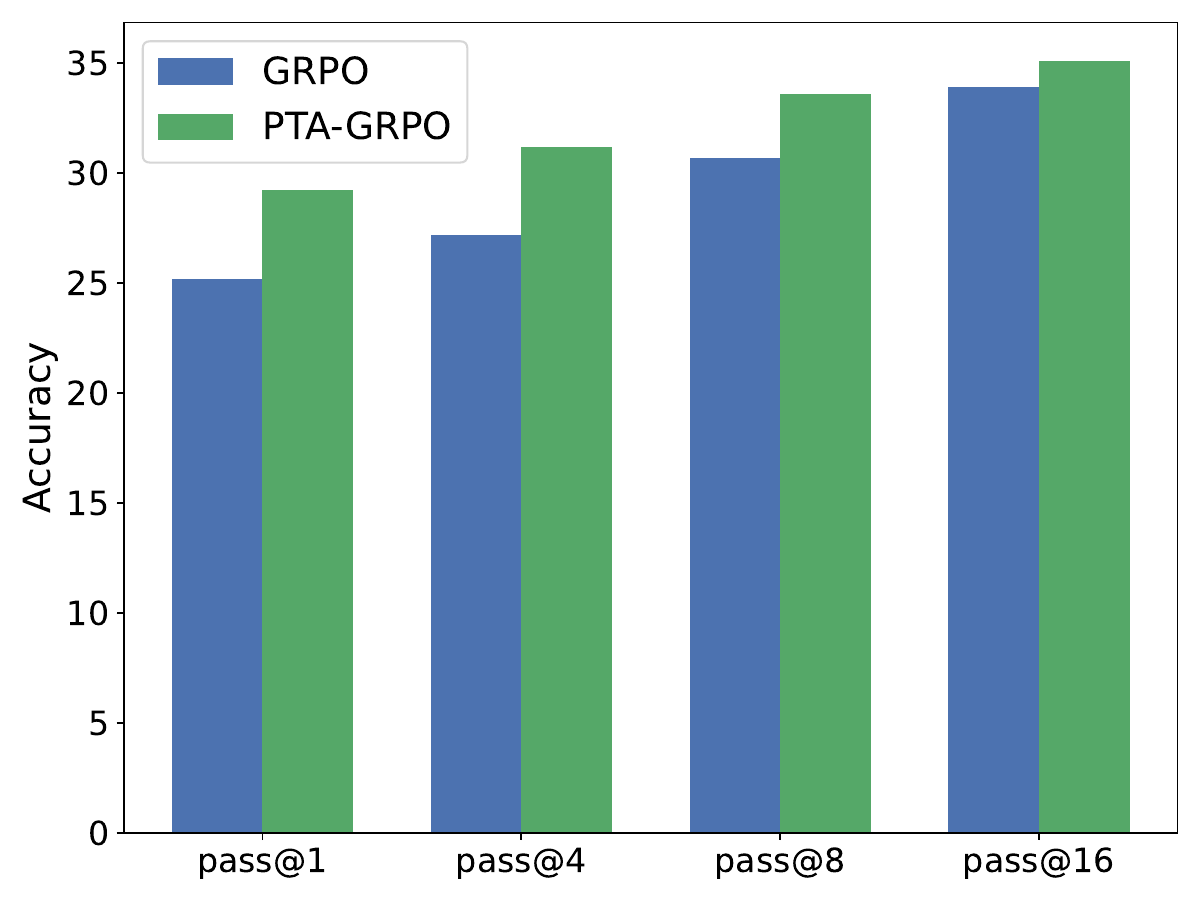}
  \caption{\small{Effect of scaling test-time compute on AIME25 (Pass@K), with Qwen2.5-7B-Instruct as the base model.}}
  \label{fig:topk}                     % 视情况微调底部空白
\end{figure}

\subsection{Sensitivity analysis}
\label{sec:sens}

From Table \ref{table:senes}, it can be observed that the parameter $\beta$ has a notable impact on the overall performance of the algorithm \alg, while the sensitivity varies across different datasets. Overall, the best Average performance (58.46) is achieved when $\beta = 0.5$, indicating that this setting provides a favorable balance across multiple tasks.

Specifically, performance on MATH500 shows a trend of first increasing and then decreasing as $\beta$ grows, reaching its peak at $\beta = 0.7$ (93.07), which suggests that this task benefits from a relatively larger $\beta$. In contrast, AIME24 and AIME25 achieve their best results at $\beta = 0.5$, and their performance drops noticeably when $\beta$ is either increased or decreased, indicating higher sensitivity to the parameter choice. For AMC23, the highest score is obtained at a smaller value, $\beta = 0.3$ (77.98), after which the performance gradually declines as $\beta$ increases, reflecting a different preference from the other datasets.

\begin{table}[h]
\caption{Sensitivity analysis of \alg}
\centering
\footnotesize
\begin{tabular}{c|ccccc}
\midrule[1.2pt]
$\beta$ & MATH500        & AIME24         & AIME25         & AMC23          & Average        \\ \hline
0.3     & 91.95          & 33.28          & 28.57          & \textbf{77.98} & 57.95          \\
0.5     & 92.53          & \textbf{34.53} & \textbf{29.25} & 77.51          & \textbf{58.46} \\
0.7     & \textbf{93.07} & 30.51          & 26.29          & 75.45          & 56.33          \\
0.9     & 89.04          & 28.82          & 27.53          & 74.48          & 54.97          \\ \midrule[1.2pt]
\end{tabular}

\label{table:senes}
\end{table}

In summary, the choice of $\beta$ involves a trade-off among different tasks: larger $\beta$ values favor MATH500, while a moderate setting ($\beta = 0.5$) yields the most robust overall performance. Therefore, when considering multiple datasets simultaneously, $\beta = 0.5$ represents a reasonable choice.

% \subsection{Sensitivity to the Number of Sampled Plans and CoT Trajectories}
To provide a more comprehensive evaluation, we conduct a detailed sensitivity analysis of the number of plans $m$ and the number of CoT trajectories $z$ sampled per plan. The full results and discussions are presented in Appendix \ref{sec:sens_m_z}.
The sensitivity results for parameter $m$ are shown in Fig.~\ref{fig:placeholder}, whereas the analysis for parameter $z$ is summarized in Table~\ref{tab:sen_cot}.

% \subsection{Sensitivity Analysis for the Number of Sampled Plans and CoT Trajectories}

% We present the sensitivity analysis of the analytic plan number $m$ and the corresponding number of CoT trajectories per plan $z$ in Appendix \ref{sec:sens_m_z}.

\section{Conclusion}
We propose Plan-Guide Enhanced Reasoning with Group Relative Policy Optimization (\alg), which integrates high-level planning with fine-grained reasoning to alleviate the lack of global planning in traditional CoT reasoning. Experimental results show that \alg achieves significant improvements across multiple mathematical reasoning benchmarks and model scales, validating its effectiveness and generalization.

% In the unusual situation where you want a paper to appear in the
% references without citing it in the main text, use \nocite
%\nocite{langley00}

\section*{Impact Statement}
This paper introduces PTA-GRPO, a two-stage training framework that explicitly optimizes high-level analytic planning in large language model (LLM) reasoning. By incorporating plan quality as a reinforcement learning objective alongside final-answer correctness, the proposed method aims to improve global reasoning coherence, reduce redundant or inconsistent reasoning trajectories, and enhance generalization across tasks and modalities.

From a positive perspective, improving structured planning and reasoning may contribute to more reliable LLM behavior in domains that require multi-step problem solving, such as mathematics, science education, programming, and multimodal reasoning. By encouraging models to produce explicit intermediate plans and better-aligned reasoning trajectories, our method may also improve interpretability compared to purely outcome-based reinforcement learning approaches.

However, stronger reasoning and planning capabilities can also amplify potential risks associated with advanced LLM systems. More effective structured reasoning may enable models to solve increasingly complex tasks, which could be misused in contexts involving academic dishonesty, automated generation of misleading content, or assistance in harmful problem-solving scenarios. Although our work focuses on benchmarked reasoning tasks in mathematics and science, the underlying techniques are general and could be applied to broader domains.

Importantly, PTA-GRPO does not introduce new data sources, external tools, or autonomous capabilities beyond standard language model training. It modifies the training objective to better align intermediate planning with correct outcomes. The method remains compatible with existing alignment, safety filtering, and content moderation techniques. We believe that responsible deployment, continued alignment research, and domain-specific safeguards are necessary to mitigate misuse risks as reasoning capabilities improve.

Overall, this work aims to advance research on structured reasoning in LLMs while acknowledging that improvements in reasoning power should be accompanied by careful consideration of downstream societal impacts and appropriate governance mechanisms.

\bibliography{example_paper}

@article{ouyang2022training,
  title={Training language models to follow instructions with human feedback},
  author={Ouyang, Long and Wu, Jeffrey and Jiang, Xu and Almeida, Diogo and Wainwright, Carroll and Mishkin, Pamela and Zhang, Chong and Agarwal, Sandhini and Slama, Katarina and Ray, Alex and others},
  journal={Advances in neural information processing systems},
  volume={35},
  pages={27730--27744},
  year={2022}
}

@article{bai2025qwen2,
  title={Qwen2. 5-vl technical report},
  author={Bai, Shuai and Chen, Keqin and Liu, Xuejing and Wang, Jialin and Ge, Wenbin and Song, Sibo and Dang, Kai and Wang, Peng and Wang, Shijie and Tang, Jun and others},
  journal={arXiv preprint arXiv:2502.13923},
  year={2025}
}

@article{hu2025open,
  title={Open-reasoner-zero: An open source approach to scaling up reinforcement learning on the base model},
  author={Hu, Jingcheng and Zhang, Yinmin and Han, Qi and Jiang, Daxin and Zhang, Xiangyu and Shum, Heung-Yeung},
  journal={arXiv preprint arXiv:2503.24290},
  year={2025}
}

@article{yu2025dapo,
  title={Dapo: An open-source llm reinforcement learning system at scale},
  author={Yu, Qiying and Zhang, Zheng and Zhu, Ruofei and Yuan, Yufeng and Zuo, Xiaochen and Yue, Yu and Fan, Tiantian and Liu, Gaohong and Liu, Lingjun and Liu, Xin and others},
  journal={arXiv preprint arXiv:2503.14476},
  year={2025}
}

@article{liu2025understanding,
  title={Understanding r1-zero-like training: A critical perspective},
  author={Liu, Zichen and Chen, Changyu and Li, Wenjun and Qi, Penghui and Pang, Tianyu and Du, Chao and Lee, Wee Sun and Lin, Min},
  journal={arXiv preprint arXiv:2503.20783},
  year={2025}
}

@article{bai2023qwen,
  title={Qwen-vl: A frontier large vision-language model with versatile abilities},
  author={Bai, Jinze and Bai, Shuai and Yang, Shusheng and Wang, Shijie and Tan, Sinan and Wang, Peng and Lin, Junyang and Zhou, Chang and Zhou, Jingren},
  journal={arXiv preprint arXiv:2308.12966},
  volume={1},
  number={2},
  pages={3},
  year={2023}
}

@article{wu2024deepseek,
  title={Deepseek-vl2: Mixture-of-experts vision-language models for advanced multimodal understanding},
  author={Wu, Zhiyu and Chen, Xiaokang and Pan, Zizheng and Liu, Xingchao and Liu, Wen and Dai, Damai and Gao, Huazuo and Ma, Yiyang and Wu, Chengyue and Wang, Bingxuan and others},
  journal={arXiv preprint arXiv:2412.10302},
  year={2024}
}

@article{wei2022chain,
  title={Chain-of-thought prompting elicits reasoning in large language models},
  author={Wei, Jason and Wang, Xuezhi and Schuurmans, Dale and Bosma, Maarten and Xia, Fei and Chi, Ed and Le, Quoc V and Zhou, Denny and others},
  journal={Advances in neural information processing systems},
  volume={35},
  pages={24824--24837},
  year={2022}
}

@article{wang2025vl,
  title={VL-Rethinker: Incentivizing Self-Reflection of Vision-Language Models with Reinforcement Learning},
  author={Wang, Haozhe and Qu, Chao and Huang, Zuming and Chu, Wei and Lin, Fangzhen and Chen, Wenhu},
  journal={arXiv preprint arXiv:2504.08837},
  year={2025}
}

@article{shah2025rethinking,
  title={Rethinking Reflection in Pre-Training},
  author={Shah, Darsh J and Rushton, Peter and Singla, Somanshu and Parmar, Mohit and Smith, Kurt and Vanjani, Yash and Vaswani, Ashish and Chaluvaraju, Adarsh and Hojel, Andrew and Ma, Andrew and others},
  journal={arXiv preprint arXiv:2504.04022},
  year={2025}
}

@article{yue2025does,
  title={Does Reinforcement Learning Really Incentivize Reasoning Capacity in LLMs Beyond the Base Model?},
  author={Yue, Yang and Chen, Zhiqi and Lu, Rui and Zhao, Andrew and Wang, Zhaokai and Song, Shiji and Huang, Gao},
  journal={arXiv preprint arXiv:2504.13837},
  year={2025}
}

@article{gandhi2025cognitive,
  title={Cognitive behaviors that enable self-improving reasoners, or, four habits of highly effective stars},
  author={Gandhi, Kanishk and Chakravarthy, Ayush and Singh, Anikait and Lile, Nathan and Goodman, Noah D},
  journal={arXiv preprint arXiv:2503.01307},
  year={2025}
}

@inproceedings{standley2023extensible,
  title={An Extensible Multi-modal Multi-task Object Dataset with Materials},
  author={Standley, Trevor and Gao, Ruohan and Chen, Dawn and Wu, Jiajun and Savarese, Silvio},
  booktitle={International Conference on Learning Representations},
  year={2023}
}

@article{guha2025openthoughts,
  title={OpenThoughts: Data Recipes for Reasoning Models},
  author={Guha, Etash and Marten, Ryan and Keh, Sedrick and Raoof, Negin and Smyrnis, Georgios and Bansal, Hritik and Nezhurina, Marianna and Mercat, Jean and Vu, Trung and Sprague, Zayne and others},
  journal={arXiv preprint arXiv:2506.04178},
  year={2025}
}

@article{yang2025qwen3,
  title={Qwen3 technical report},
  author={Yang, An and Li, Anfeng and Yang, Baosong and Zhang, Beichen and Hui, Binyuan and Zheng, Bo and Yu, Bowen and Gao, Chang and Huang, Chengen and Lv, Chenxu and others},
  journal={arXiv preprint arXiv:2505.09388},
  year={2025}
}

@article{he2025deepmath,
  title={Deepmath-103k: A large-scale, challenging, decontaminated, and verifiable mathematical dataset for advancing reasoning},
  author={He, Zhiwei and Liang, Tian and Xu, Jiahao and Liu, Qiuzhi and Chen, Xingyu and Wang, Yue and Song, Linfeng and Yu, Dian and Liang, Zhenwen and Wang, Wenxuan and others},
  journal={arXiv preprint arXiv:2504.11456},
  year={2025}
}

@article{wan2025srpo,
  title={Srpo: Enhancing multimodal llm reasoning via reflection-aware reinforcement learning},
  author={Wan, Zhongwei and Dou, Zhihao and Liu, Che and Zhang, Yu and Cui, Dongfei and Zhao, Qinjian and Shen, Hui and Xiong, Jing and Xin, Yi and Jiang, Yifan and others},
  journal={arXiv preprint arXiv:2506.01713},
  year={2025}
}

@article{schulman2017proximal,
  title={Proximal policy optimization algorithms},
  author={Schulman, John and Wolski, Filip and Dhariwal, Prafulla and Radford, Alec and Klimov, Oleg},
  journal={arXiv preprint arXiv:1707.06347},
  year={2017}
}

@article{kahneman2011thinking,
  title={Thinking, fast and slow},
  author={Kahneman, Daniel},
  journal={Farrar, Straus and Giroux},
  year={2011}
}

@incollection{kahneman2013prospect,
 title     = {Prospect Theory: An Analysis of Decision under Risk},
  author    = {Kahneman, Daniel and Tversky, Amos},
  booktitle = {Handbook of the Fundamentals of Financial Decision Making: Part I},
  pages     = {99--127},
  year      = {2013},
  publisher = {World Scientific}
}

@article{lu2025arpo,
  title={ARPO: End-to-End Policy Optimization for GUI Agents with Experience Replay},
  author={Lu, Fanbin and Zhong, Zhisheng and Liu, Shu and Fu, Chi-Wing and Jia, Jiaya},
  journal={arXiv preprint arXiv:2505.16282},
  year={2025}
}

@article{xu2025towards,
  title={Towards large reasoning models: A survey of reinforced reasoning with large language models},
  author={Xu, Fengli and Hao, Qianyue and Zong, Zefang and Wang, Jingwei and Zhang, Yunke and Wang, Jingyi and Lan, Xiaochong and Gong, Jiahui and Ouyang, Tianjian and Meng, Fanjin and others},
  journal={arXiv preprint arXiv:2501.09686},
  year={2025}
}

@article{plaat2024reasoning,
  title={Reasoning with large language models, a survey},
  author={Plaat, Aske and Wong, Annie and Verberne, Suzan and Broekens, Joost and van Stein, Niki and B{\"a}ck, Thomas},
  journal={CoRR},
  year={2024}
}

@article{ke2025survey,
  title={A survey of frontiers in llm reasoning: Inference scaling, learning to reason, and agentic systems},
  author={Ke, Zixuan and Jiao, Fangkai and Ming, Yifei and Nguyen, Xuan-Phi and Xu, Austin and Long, Do Xuan and Li, Minzhi and Qin, Chengwei and Wang, Peifeng and Savarese, Silvio and others},
  journal={arXiv preprint arXiv:2504.09037},
  year={2025}
}

@article{zhang2024llama,
  title={Llama-berry: Pairwise optimization for o1-like olympiad-level mathematical reasoning},
  author={Zhang, Di and Wu, Jianbo and Lei, Jingdi and Che, Tong and Li, Jiatong and Xie, Tong and Huang, Xiaoshui and Zhang, Shufei and Pavone, Marco and Li, Yuqiang and others},
  journal={arXiv preprint arXiv:2410.02884},
  year={2024}
}

@article{liu2023improving,
  title={Improving large language model fine-tuning for solving math problems},
  author={Liu, Yixin and Singh, Avi and Freeman, C Daniel and Co-Reyes, John D and Liu, Peter J},
  journal={arXiv preprint arXiv:2310.10047},
  year={2023}
}

@article{wu2023mathchat,
  title={Mathchat: Converse to tackle challenging math problems with llm agents},
  author={Wu, Yiran and Jia, Feiran and Zhang, Shaokun and Li, Hangyu and Zhu, Erkang and Wang, Yue and Lee, Yin Tat and Peng, Richard and Wu, Qingyun and Wang, Chi},
  journal={ICLR 2024 Workshop on LLM Agents},
  year={2024}
}

@article{jiang2024survey,
  title={A survey on large language models for code generation},
  author={Jiang, Juyong and Wang, Fan and Shen, Jiasi and Kim, Sungju and Kim, Sunghun},
  journal={arXiv preprint arXiv:2406.00515},
  year={2024}
}

@article{yao2023tree,
  title={Tree of thoughts: Deliberate problem solving with large language models},
  author={Yao, Shunyu and Yu, Dian and Zhao, Jeffrey and Shafran, Izhak and Griffiths, Tom and Cao, Yuan and Narasimhan, Karthik},
  journal={Advances in neural information processing systems},
  volume={36},
  pages={11809--11822},
  year={2023}
}

@article{wang2024litesearch,
  title={Litesearch: Efficacious tree search for llm},
  author={Wang, Ante and Song, Linfeng and Tian, Ye and Peng, Baolin and Yu, Dian and Mi, Haitao and Su, Jinsong and Yu, Dong},
  journal={arXiv preprint arXiv:2407.00320},
  year={2024}
}

@article{feng2023alphazero,
  title={Alphazero-like tree-search can guide large language model decoding and training},
  author={Feng, Xidong and Wan, Ziyu and Wen, Muning and McAleer, Stephen Marcus and Wen, Ying and Zhang, Weinan and Wang, Jun},
  journal={arXiv preprint arXiv:2309.17179},
  year={2023}
}

@article{qian2025demystifying,
  title={Demystifying reasoning dynamics with mutual information: Thinking tokens are information peaks in llm reasoning},
  author={Qian, Chen and Liu, Dongrui and Wen, Haochen and Bai, Zhen and Liu, Yong and Shao, Jing},
  journal={arXiv preprint arXiv:2506.02867},
  year={2025}
}

@article{shao2024deepseekmath,
  title={Deepseekmath: Pushing the limits of mathematical reasoning in open language models},
  author={Shao, Zhihong and Wang, Peiyi and Zhu, Qihao and Xu, Runxin and Song, Junxiao and Bi, Xiao and Zhang, Haowei and Zhang, Mingchuan and Li, YK and Wu, Yang and others},
  journal={arXiv preprint arXiv:2402.03300},
  year={2024}
}

@article{qu2025survey,
  title={A survey of efficient reasoning for large reasoning models: Language, multimodality, and beyond},
  author={Qu, Xiaoye and Li, Yafu and Su, Zhaochen and Sun, Weigao and Yan, Jianhao and Liu, Dongrui and Cui, Ganqu and Liu, Daizong and Liang, Shuxian and He, Junxian and others},
  journal={arXiv preprint arXiv:2503.21614},
  year={2025}
}

@article{dubey2024llama,
  title={The llama 3 herd of models},
  author={Dubey, Abhimanyu and Jauhri, Abhinav and Pandey, Abhinav and Kadian, Abhishek and Al-Dahle, Ahmad and Letman, Aiesha and Mathur, Akhil and Schelten, Alan and Yang, Amy and Fan, Angela and others},
  journal={arXiv e-prints},
  pages={arXiv--2407},
  year={2024}
}

@techreport{openai_gpt5_systemcard,
  title        = {GPT-5 System Card},
  author       = {OpenAI},
  year         = {2025},
  note         = {Accessed: 2025-08-13},
  url          = {https://cdn.openai.com/gpt-5-system-card.pdf}
}

@article{seed2025seed1,
  title={Seed1. 5-thinking: Advancing superb reasoning models with reinforcement learning},
  author={Seed, ByteDance and Chen, Jiaze and Fan, Tiantian and Liu, Xin and Liu, Lingjun and Lin, Zhiqi and Wang, Mingxuan and Wang, Chengyi and Wei, Xiangpeng and Xu, Wenyuan and others},
  journal={arXiv preprint arXiv:2504.13914},
  year={2025}
}

@article{li2025treepo,
  title={Treepo: Bridging the gap of policy optimization and efficacy and inference efficiency with heuristic tree-based modeling},
  author={Li, Yizhi and Gu, Qingshui and Wen, Zhoufutu and Li, Ziniu and Xing, Tianshun and Guo, Shuyue and Zheng, Tianyu and Zhou, Xin and Qu, Xingwei and Zhou, Wangchunshu and others},
  journal={arXiv preprint arXiv:2508.17445},
  year={2025}
}

@article{zhang2025srpo,
  title={Srpo: A cross-domain implementation of large-scale reinforcement learning on llm},
  author={Zhang, Xiaojiang and Wang, Jinghui and Cheng, Zifei and Zhuang, Wenhao and Lin, Zheng and Zhang, Minglei and Wang, Shaojie and Cui, Yinghan and Wang, Chao and Peng, Junyi and others},
  journal={arXiv preprint arXiv:2504.14286},
  year={2025}
}

@article{zhang2025right,
  title={Right question is already half the answer: Fully unsupervised llm reasoning incentivization},
  author={Zhang, Qingyang and Wu, Haitao and Zhang, Changqing and Zhao, Peilin and Bian, Yatao},
  journal={arXiv preprint arXiv:2504.05812},
  year={2025}
}

@article{feng2025group,
  title={Group-in-group policy optimization for llm agent training},
  author={Feng, Lang and Xue, Zhenghai and Liu, Tingcong and An, Bo},
  journal={arXiv preprint arXiv:2505.10978},
  year={2025}
}

@article{wang2024cpl,
  title={Cpl: Critical plan step learning boosts llm generalization in reasoning tasks},
  author={Wang, Tianlong and Chen, Junzhe and Han, Xueting and Bai, Jing},
  journal={arXiv preprint arXiv:2409.08642},
  year={2024}
}

@inproceedings{xu2025full,
  title={Full-step-dpo: Self-supervised preference optimization with step-wise rewards for mathematical reasoning},
  author={Xu, Huimin and Mao, Xinnian and Li, Feng-Lin and Wu, Xiaobao and Chen, Wang and Zhang, Wei and Tuan, Luu Anh},
  booktitle={Findings of the Association for Computational Linguistics: ACL 2025},
  pages={24343--24356},
  year={2025}
}

@inproceedings{yue2025mmmu,
  title={Mmmu-pro: A more robust multi-discipline multimodal understanding benchmark},
  author={Yue, Xiang and Zheng, Tianyu and Ni, Yuansheng and Wang, Yubo and Zhang, Kai and Tong, Shengbang and Sun, Yuxuan and Yu, Botao and Zhang, Ge and Sun, Huan and others},
  booktitle={Proceedings of the 63rd Annual Meeting of the Association for Computational Linguistics (Volume 1: Long Papers)},
  pages={15134--15186},
  year={2025}
}

@article{meng2025mm,
  title={Mm-eureka: Exploring the frontiers of multimodal reasoning with rule-based reinforcement learning},
  author={Meng, Fanqing and Du, Lingxiao and Liu, Zongkai and Zhou, Zhixiang and Lu, Quanfeng and Fu, Daocheng and Han, Tiancheng and Shi, Botian and Wang, Wenhai and He, Junjun and others},
  journal={arXiv preprint arXiv:2503.07365},
  year={2025}
}

@inproceedings{yue2024mmmu,
  title={Mmmu: A massive multi-discipline multimodal understanding and reasoning benchmark for expert agi},
  author={Yue, Xiang and Ni, Yuansheng and Zhang, Kai and Zheng, Tianyu and Liu, Ruoqi and Zhang, Ge and Stevens, Samuel and Jiang, Dongfu and Ren, Weiming and Sun, Yuxuan and others},
  booktitle={Proceedings of the IEEE/CVF Conference on Computer Vision and Pattern Recognition},
  pages={9556--9567},
  year={2024}
}

@inproceedings{yao2022react,
  title={React: Synergizing reasoning and acting in language models},
  author={Yao, Shunyu and Zhao, Jeffrey and Yu, Dian and Du, Nan and Shafran, Izhak and Narasimhan, Karthik R and Cao, Yuan},
  booktitle={The eleventh international conference on learning representations},
  year={2022}
}

@article{wang2023plan,
  title={Plan-and-solve prompting: Improving zero-shot chain-of-thought reasoning by large language models},
  author={Wang, Lei and Xu, Wanyu and Lan, Yihuai and Hu, Zhiqiang and Lan, Yunshi and Lee, Roy Ka-Wei and Lim, Ee-Peng},
  journal={arXiv preprint arXiv:2305.04091},
  year={2023}
}

@article{zhou2022least,
  title={Least-to-most prompting enables complex reasoning in large language models},
  author={Zhou, Denny and Sch{\"a}rli, Nathanael and Hou, Le and Wei, Jason and Scales, Nathan and Wang, Xuezhi and Schuurmans, Dale and Cui, Claire and Bousquet, Olivier and Le, Quoc and others},
  journal={ICLR},
  year={2023}
}

@article{wan2025rema,
  title={Rema: Learning to meta-think for llms with multi-agent reinforcement learning},
  author={Wan, Ziyu and Li, Yunxiang and Wen, Xiaoyu and Song, Yan and Wang, Hanjing and Yang, Linyi and Schmidt, Mark and Wang, Jun and Zhang, Weinan and Hu, Shuyue and others},
  journal={NeurIPS},
  year={2025}
}

@article{erdogan2025plan,
  title={Plan-and-act: Improving planning of agents for long-horizon tasks},
  author={Erdogan, Lutfi Eren and Lee, Nicholas and Kim, Sehoon and Moon, Suhong and Furuta, Hiroki and Anumanchipalli, Gopala and Keutzer, Kurt and Gholami, Amir},
  journal={ICML},
  year={2025}
}

@article{chen2025seed,
  title={Seed-grpo: Semantic entropy enhanced grpo for uncertainty-aware policy optimization},
  author={Chen, Minghan and Chen, Guikun and Wang, Wenguan and Yang, Yi},
  journal={arXiv preprint arXiv:2505.12346},
  year={2025}
}

@article{liang2026render,
  title={Render-in-the-Loop: Vector Graphics Generation via Visual Self-Feedback},
  author={Liang, Guotao and Wang, Zhangcheng and Hu, Juncheng and Zhou, Haitao and Xue, Ziteng and Zhang, Jing and Xu, Dong and Yu, Qian},
  journal={arXiv preprint arXiv:2604.20730},
  year={2026}
}

@article{dou2025dsadf,
  title={Dsadf: Thinking fast and slow for decision making},
  author={Dou, Zhihao and Cui, Dongfei and Yan, Jun and Wang, Weida and Chen, Benteng and Wang, Haoming and Xie, Zeke and Zhang, Shufei},
  journal={International Journal of Computer Vision},
  volume={134},
  number={6},
  pages={270},
  year={2026},
  publisher={Springer}
}

@article{liang2026vanim,
  title={VAnim: Rendering-Aware Sparse State Modeling for Structure-Preserving Vector Animation},
  author={Liang, Guotao and Wang, Zhangcheng and Wang, Chuang and Hu, Juncheng and Zhou, Haitao and Liu, Junhua and Zhang, Jing and Xu, Dong and Yu, Qian},
  journal={arXiv preprint arXiv:2605.01517},
  year={2026}
}
\bibliographystyle{icml2026}
\balance

%%%%%%%%%%%%%%%%%%%%%%%%%%%%%%%%%%%%%%%%%%%%%%%%%%%%%%%%%%%%%%%%%%%%%%%%%%%%%%%
%%%%%%%%%%%%%%%%%%%%%%%%%%%%%%%%%%%%%%%%%%%%%%%%%%%%%%%%%%%%%%%%%%%%%%%%%%%%%%%
% APPENDIX
%%%%%%%%%%%%%%%%%%%%%%%%%%%%%%%%%%%%%%%%%%%%%%%%%%%%%%%%%%%%%%%%%%%%%%%%%%%%%%%
%%%%%%%%%%%%%%%%%%%%%%%%%%%%%%%%%%%%%%%%%%%%%%%%%%%%%%%%%%%%%%%%%%%%%%%%%%%%%%%
\newpage
\appendix 
\onecolumn

\appendix
\clearpage

\section{Appendix}

\subsection{Code Implementation}

The code can be founded at \href{https://anonymous.4open.science/r/icml26-1EE8/}{\textbf{\textcolor{red}{Code}}}.

\subsection{Details of Format Reward}
\label{sec:format}

\myparatights{Format Reward}
The format reward $r_{\text{format}}$ is designed to regulate the overall structure of the model response, ensuring both conformity to the desired format and control over the output length. 
It consists of two components: $r_{\text{structure}}$ and $r_{\text{length}}$. 
Specifically, $r_{\text{structure}}$ enforces that the policy model’s response adheres to the predefined structural template, i.e., 
\texttt{<plan>...</plan>}, \texttt{<think>...</think>}, and \texttt{<answer>...</answer>}. 
Meanwhile, $r_{\text{length}}$ serves as an auxiliary reward that encourages the model to generate concise and efficient token sequences, thereby reducing redundant or uninformative content.

To provide a clearer illustration of each reward, we present its detailed formulation as follows. 
We begin with the format reward $r_{\text{format}}$, which is defined as: 
\begin{equation}
r_{\text{format}} = 
\begin{cases}
0.2, & \text{if the response strictly} \\ &\text{ follows the predefined template} \\
0, & \text{otherwise}.
\end{cases}
\end{equation}
This function enforces a binary constraint on the output structure: a full reward is granted only when the response strictly adheres to the predefined template, thereby ensuring the consistency and parsability of the generated results.

For response length, the optimal number of tokens varies across different questions, 
making it difficult to predefine a fixed target length. 
Therefore, for all responses generated for a given question, 
we select the shortest correct response length as the reference length $T$, defined as:
\begin{equation}
T = \min \{\, |\{t_i,c_{i,k}\}| \;\mid\; \hat{y}_{i,k} = y \,\},
\end{equation}
where $|\{t_i,c_{i,k}\}|$ denotes the token length of response $\{t_i,c_{i,k}\}$.
Here, $T$ represents the shortest executable token length required to obtain the correct answer to a given question. 
It can be regarded as the optimal reference length under current knowledge, 
toward which other correct responses should converge in order to minimize redundancy while preserving correctness.
For each response $\{t_i,c_{i,k}\} \in G$, the length reward $r_{\text{length}}$ can be expressed as:
\begin{equation}
r_{\text{length}}(\{t_i,c_{i,k}\}) = \alpha \cdot \text{exp}(- \frac{\left| |\{t_i,c_{i,k}\}|-T\right|}{T_{\text{max}}-T}),
\end{equation}
where $\alpha$ is a hyperparameter, and $T_{\text{max}}$ does not denote the maximum output length set for the policy model. The reward becomes larger as the response length approaches the reference length $T$, encouraging the model to generate concise yet correct responses. 

The format reward $r_{\text{format}}$, defined as $r_{\text{format}} = r_{\text{structure}} + r_{\text{length}}$, ensures that the output not only adheres to the required format, but also guarantees the conciseness of the output response.

\subsection{RL training time for various RLVR approaches}
\label{sec:time}

Using Qwen2.5-7B-Instruct as the base model, the training time of GRPO is 44.7 hours, DAPO requires 47.5 hours, and \alg takes 44.9 hours. When using Qwen3-8B as the base model, GRPO requires 59.7 hours, DAPO 66.9 hours, and \alg 61.4 hours. These results indicate that PTA-GRPO does not introduce a noticeable increase in training time compared to existing RLVR-based methods.

\subsection{Impact Results of Self-Generated Plans in a Policy Model}
\label{sec:self}

\begin{table*}[]
\centering
\caption{Comparative Analysis of Self-Generated Plans in a Policy Model}
\begin{tabular}{c|ccccc}
\midrule[1.2pt]
Method                            & MATH500        & AIME24         & AIME25         & AMC23          & Average        \\ \hline
\alg  w/ self-plan & 92.18 & 32.79 & 28.57 & 75.94 & 57.37               \\
\alg               & 92.53 & {34.53} & 29.25 & {77.51} & {58.46} \\ \midrule[1.2pt]
\end{tabular}
\label{tab:self}
\end{table*}

Table~\ref{tab:self} compares the performance of \alg with and without self-generated plans, where `\alg w/ self-plan` denotes that the plan is summarized and generated solely by the policy model without relying on an advanced LLM for planning. The results show that the two settings exhibit highly consistent performance across all benchmarks: the difference on MATH500 is negligible, while slight performance gaps appear on more challenging tasks such as AIME24, AIME25, and AMC23, though the margins remain limited. In terms of average performance, `\alg w/ self-plan` achieves a score of 57.37, which is close to the full version's 58.46. These findings indicate that even without explicit plan generation from an advanced LLM, \alg can maintain stable and competitive performance by leveraging the policy model’s self-generated planning capability, demonstrating the intrinsic effectiveness and practical robustness of the proposed approach.

\subsection{Training Dynamics of \alg}
 \label{sec:training dynamic}
 
 Fig.~\ref{fig:qwen3-8b} and Fig.~\ref{fig:qwen25-7b} illustrate the training dynamics of QWEN3-8B and QWEN2.5-7B-Instruct, respectively. As shown in the figures, our method outperforms GRPO in terms of accuracy reward and response length, indicating the effectiveness of the introduced component. It is worth noting that in Fig.~\ref{fig:qwen3-8b} (b), our approach achieves lower entropy compared to GRPO. This suggests that for stronger models, our method encourages the development of more reasonable analytic plans, enabling the model to complete a given trajectory with greater confidence and ultimately achieving higher accuracy.

\begin{figure}[H]
    \centering
    % 三个子图并排
    \begin{subfigure}{0.32\textwidth}
        \centering
        \includegraphics[width=\linewidth]{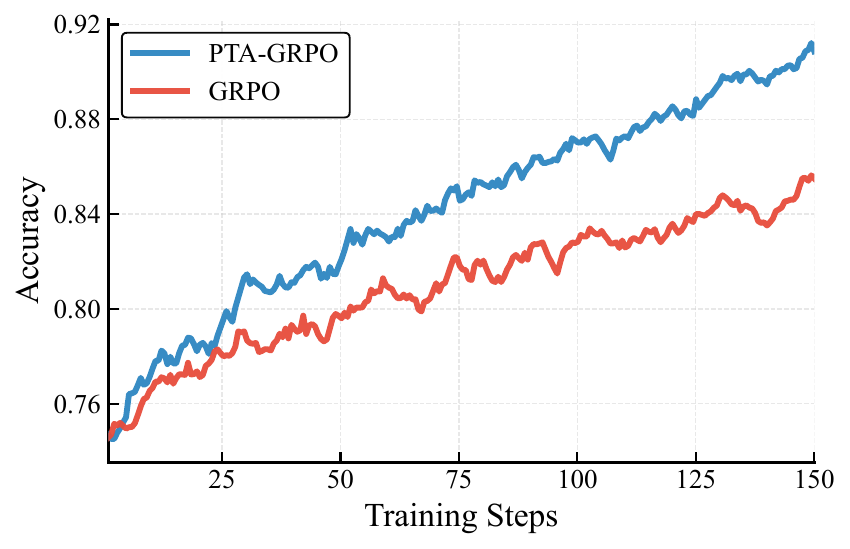}
        \caption{Accuracy Reward}
    \end{subfigure}
    \hfill
    \begin{subfigure}{0.32\textwidth}
        \centering
        \includegraphics[width=\linewidth]{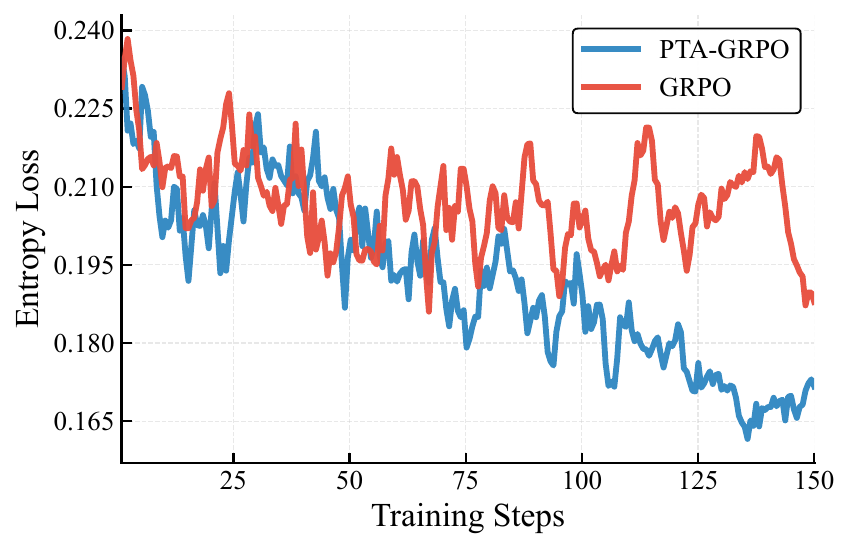}
        \label{fig:8b_entro}
        \caption{Policy Entropy}
    \end{subfigure}
    \hfill
    \begin{subfigure}{0.32\textwidth}
        \centering
        \includegraphics[width=\linewidth]{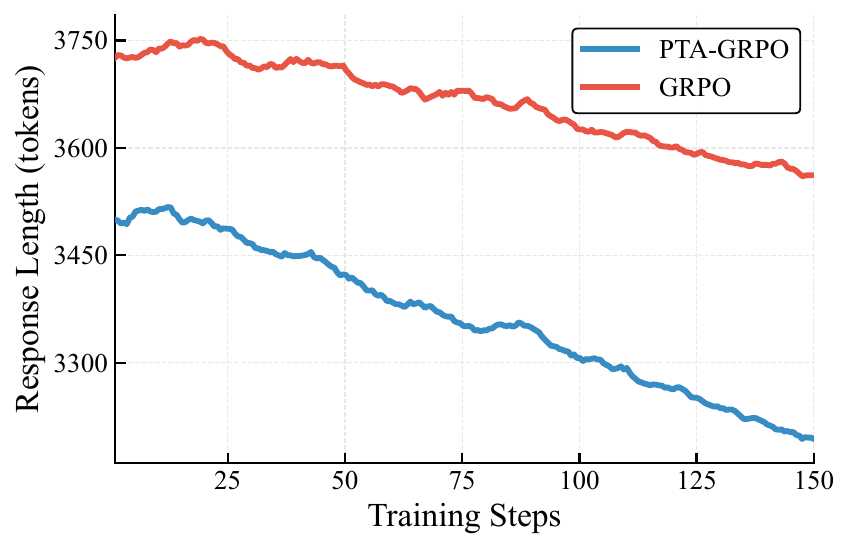}
        \caption{Response Length}
    \end{subfigure}

    \caption{\small{Training Dynamics of \alg with Qwen3-8B.}}
    
    \label{fig:qwen3-8b}
\end{figure}

% -------- Qwen2.5-7B-Instruct --------
\begin{figure}[H]
    \centering
    % 三个子图在同一行
    \begin{subfigure}{0.32\textwidth}
        \centering
        \includegraphics[width=\linewidth]{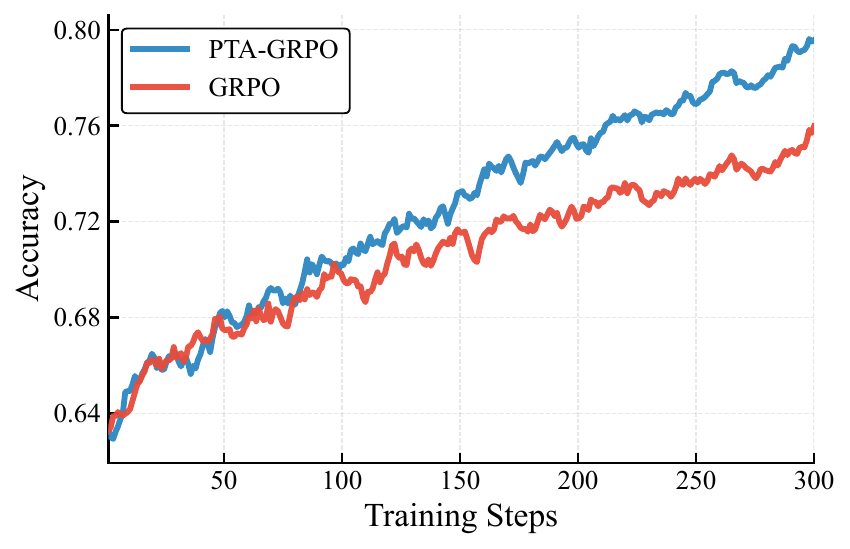}
        \caption{Accuracy Reward}
    \end{subfigure}
    \hfill
    \begin{subfigure}{0.32\textwidth}
        \centering
        \includegraphics[width=\linewidth]{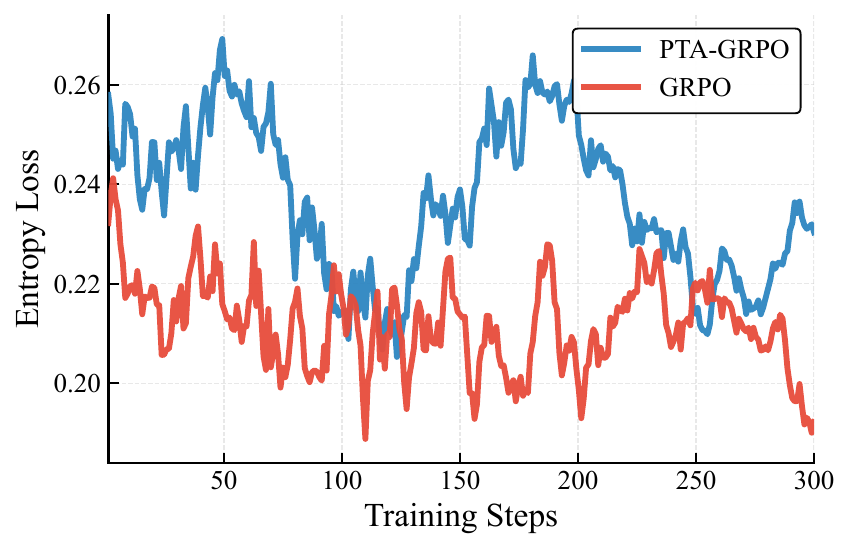}
        \caption{Policy Entropy}
    \end{subfigure}
    \hfill
    \begin{subfigure}{0.32\textwidth}
        \centering
        \includegraphics[width=\linewidth]{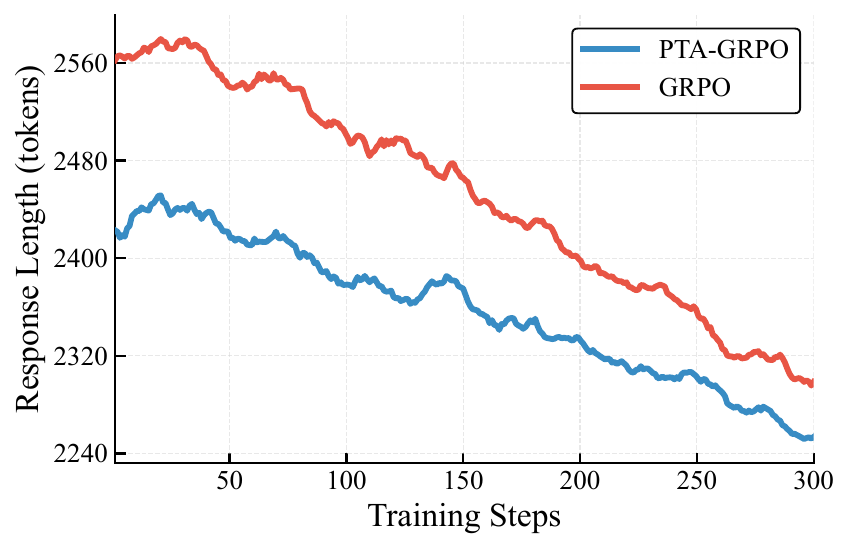}
        \caption{Response Length}
    \end{subfigure}

    \caption{\small{Training Dynamics of \alg with Qwen2.5-7B-Instruct.}  }
    \label{fig:qwen25-7b}
\end{figure}

\subsection{Sensitivity analysis of sampled number}
\label{sec:sens_m_z}

\begin{figure}[H]
    \centering
    \includegraphics[width=0.95\linewidth]{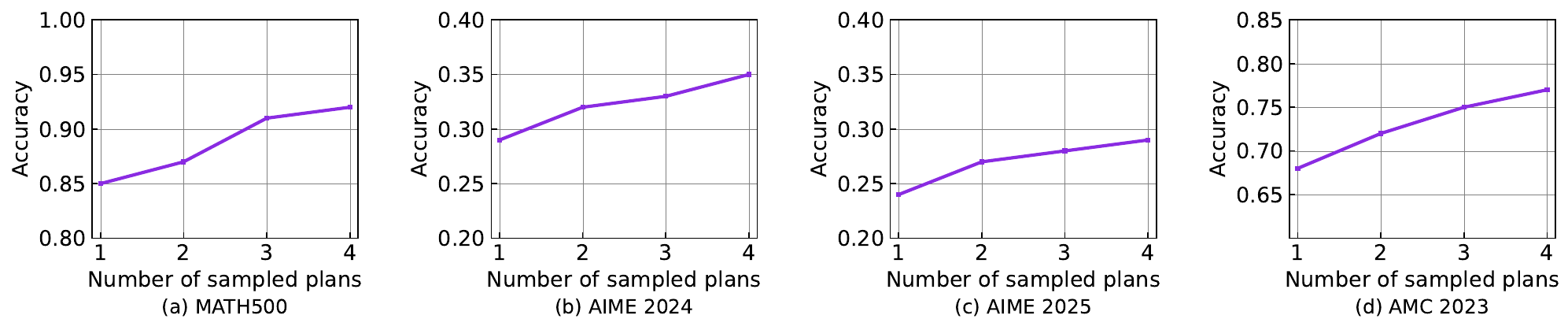}
    \caption{Sensitivity analysis with respect to the number of analytic plans $m$, where Qwen2.5-7B-Instruct is considered as based model. }
    \label{fig:placeholder}
\end{figure}

\begin{table}[H]
\centering
\caption{Sensitivity analysis with respect to the number of sampled CoTs $z$ for per analytic plan, using Qwen2.5-7B-Instruct as the base model.}
\begin{tabular}{c|ccccc}
\midrule[1.2pt]
$z$ & MATH500 & AIME 2024 & AIME 2025 & AMC 23 & Average \\ \hline
1   & 85.79   & 29.21     & 25.95     & 69.97  & 52.73   \\
2   & 89.29   & 33.29     & 27.77     & 74.73  & 56.27   \\
3   & 92.53   & 34.53     & 29.25     & 77.51  & 58.46   \\ \midrule[1.2pt]
\end{tabular}
\label{tab:sen_cot}
\end{table}

\subsection{Analysis of Statistical Significance}
\label{sec:sta_ana}

Table \ref{tab:pairwise_grpo_50k_est} reports pairwise significance tests comparing PTA-GRPO with the GRPO baseline across multiple models and mathematical reasoning benchmarks. Overall, PTA-GRPO yields consistent and statistically significant improvements, with 13 out of 16 model–task pairs achieving significance at the 5\% level or better. The gains are particularly pronounced for smaller models, such as Qwen2.5-7B and LLaMA3.2-3B, where PTA-GRPO improves performance across all tasks with large effect sizes and strong statistical significance (all $p<0.001$). For larger models (Qwen3-8B and Qwen3-14B), the improvements remain positive on most benchmarks but become more modest, and in some cases statistically insignificant, indicating diminishing returns as model capacity increases. Notably, the benefits of PTA-GRPO are most consistent on harder benchmarks (AIME24 and AIME25), where all evaluated models exhibit significant gains, suggesting that PTA-GRPO is particularly effective for complex, long-horizon mathematical reasoning. In contrast, on relatively easier tasks such as AMC23 and MATH500, the improvements are smaller and occasionally negligible, reflecting a saturation effect rather than a systematic degradation.

\begin{table}[h]
\centering
\caption{\small{
Pairwise significance tests between PTA-GRPO and GRPO.
Each row reports the improvement $\Delta$ (in percentage points), the $p$-value, and the corresponding significance level. $^*$, $^{**}$, and $^{***}$ denote $p<0.05$, $p<0.01$, and $p<0.001$, respectively; ``ns'' indicates not significant. ($p$-values are re-estimated using a normal-approximation scaling with $\sqrt{n}$ and $\Delta$.)
}}
\renewcommand{\arraystretch}{1.15}
\scalebox{0.60}{
\begin{tabular}{l|l|l|c|c}
\midrule[1.2pt]
\textbf{Model} & \textbf{Baseline} & \textbf{Task} & \textbf{$\Delta$ (pt)} & \textbf{$p$ / sig} \\
\midrule
% ========================= Qwen2.5-7B ============================
\multirow{4}{*}{\textbf{Qwen2.5-7B}}
& GRPO & MATH500 & +1.86 & 0.0000$^{***}$ \\
& GRPO & AIME24  & +4.38 & 0.0000$^{***}$ \\
& GRPO & AIME25  & +4.05 & 0.0000$^{***}$ \\
& GRPO & AMC23   & +4.30 & 0.0000$^{***}$ \\
% ========================= LLaMA3.2-3B ============================
\multirow{4}{*}{\textbf{LLaMA3.2-3B}}
& GRPO & MATH500 & +4.02 & 0.0000$^{***}$ \\
& GRPO & AIME24  & +5.37 & 0.0000$^{***}$ \\
& GRPO & AIME25  & +2.46 & 0.0000$^{***}$ \\
& GRPO & AMC23   & +3.11 & 0.0000$^{***}$ \\
% ========================= Qwen3-8B ============================
\multirow{4}{*}{\textbf{Qwen3-8B}}
& GRPO & MATH500 & -0.54 & 0.0275$^{*}$ \\
& GRPO & AIME24  & +1.11 & 0.0012$^{**}$ \\
& GRPO & AIME25  & +1.20 & 0.0003$^{***}$ \\
& GRPO & AMC23   & +0.26 & 0.4090$^{ns}$ \\
% ========================= Qwen3-14B ============================
\multirow{4}{*}{\textbf{Qwen3-14B}}
& GRPO & MATH500 & +3.67 & 0.0000$^{***}$ \\
& GRPO & AIME24  & +1.82 & 0.0000$^{***}$ \\
& GRPO & AIME25  & +1.96 & 0.0000$^{***}$ \\
& GRPO & AMC23   & +0.30 & 0.0741$^{ns}$ \\
\midrule[1.2pt]
\end{tabular}
}
\label{tab:pairwise_grpo_50k_est}
\end{table}

\subsection{Group Relative Policy Optimization (GRPO)}
\label{sec:grpo}

Group Relative Policy Optimization (GRPO) is a state-of-the-art Reinforcement Learning with Verifiable Rewards (RLVR) algorithm that simplifies Proximal Policy Optimization (PPO) \citep{schulman2017proximal} by removing the need for a value model to estimate the baseline advantage, and has demonstrated remarkable success in enhancing the reasoning abilities of LLM.
Formally, let $Q$ denote the set of questions, $\pi_{\theta_{\text{old}}}$ be the current policy model, and $\{\mathfrak{o}_i\}_{i=1}^N$ represent a collection of $N$ candidate responses sampled for a question $q \in Q$. We also define $\pi_{\theta_{\text{ref}}}$ as a fixed reference model. The training objective of GRPO is expressed as:

\begin{equation}
\begin{aligned}
& J_{\text{GRPO}}(\theta)
= \mathbb{E}_{q \sim Q,\; \{\mathfrak{o}_i\}_{i=1}^N \sim \pi_{\theta_{\text{old}}}}
\\
&\Bigl[
\frac{1}{N} \sum_{i=1}^N \sum_{t=1}^{|\mathfrak{o}_i|}
\min\Bigl(
\frac{\pi_{\theta}(\mathfrak{o}_{i}^t|q)}{\pi_{\theta_{\text{old}}}(\mathfrak{o}_{i}^t|q)} A_i,\,
\text{clip}\Bigl(
\frac{\pi_{\theta}(\mathfrak{o}_{i}^t|q)}{\pi_{\theta_{\text{old}}}(\mathfrak{o}_{i}^t|q)},
1-\epsilon, 1+\epsilon
\Bigr) A_i
\Bigr)
\\
&\qquad
- \beta D_{\text{KL}}(\pi_\theta \| \pi_{\text{ref}})
\Bigr]
\end{aligned}
\label{grpo:fun}
\end{equation}

where $\epsilon$ controls the clipping range and $\beta$ weights the KL regularization term.
The normalized advantage $A_i$ assigned to each response $\mathfrak{o}_i$ is computed from group-based rewards:

\begin{equation}
A_i = \frac{r_i - \mu}{\sigma},
\quad \text{with } \mu = \frac{1}{N}\sum_{j=1} r_j,
\quad \sigma = \sqrt{\frac{1}{N}\sum_{j=1}^N (r_j - \mu)^2},
\label{advantages}
\end{equation}
where $\{r_1, r_2, \dots, r_N\}$ are the scalar rewards associated with the response group $\{\mathfrak{o}_i\}_{i=1}^N$.

In GRPO, each response $\mathfrak{o} \in \{\mathfrak{o}_i\}_{i=1}^N$ consists of a CoT $c$ together with its final answer. As noted in Section~\ref{sub:reason}, token-level MDPs lack global planning and often yield redundant steps, while GRPO rewards $r$ corresponding to $\mathfrak{o}$ focus only on final answer correctness, overlooking reasoning quality and enabling reward hacking through superficial or verbose CoTs.

\subsection{Theoretical proof}
\label{sec:proof}

\begin{proof}
We consider the binary setting, where the ground-truth answer
$y\in\{0,1\}$ and $\hat y$ denotes a (possibly stochastic) predictor based on $(t,q)$. We set $0$ is wrong and $1$ is correct.
All logarithms are taken base $2$.

Fix $(t,q)$. Define the conditional (Bayes) error probability
\[
p_c(t,q) := \Pr(\tilde y \neq y \mid t,q),
\]
where $\tilde y$ denotes the Bayes-optimal (MAP) decision rule given $(\hat y,t,q)$, i.e.,
\[
\tilde y := \arg\max_{y'\in\{0,1\}} \Pr(y=y' \mid \hat y,t,q).
\]
For convenience, define the conditional posterior
\[
\eta(\hat y;t,q) := \Pr(y=1 \mid \hat y,t,q).
\]
Since $y$ is binary, the conditional Bayes error given $(\hat y,t,q)$ is
\[
\Pr(\tilde y \neq y \mid \hat y,t,q)
=
\min\{\eta(\hat y;t,q),\,1-\eta(\hat y;t,q)\}.
\]

Let $h_2(p) := -p\log p-(1-p)\log(1-p)$ denote the binary entropy function.
For any $p\in[0,1]$, it holds that
\[
h_2(p)\ge 2\min\{p,1-p\}.
\]
Applying this inequality with $p=\eta(\hat y;t,q)$ yields
\[
2\min\{\eta(\hat y;t,q),1-\eta(\hat y;t,q)\}
\le
h_2(\eta(\hat y;t,q)).
\]
Taking expectation over $\hat y$ conditional on $(t,q)$, we obtain
\[
p_c(t,q)
\le
\frac12\,\E\!\left[h_2(\eta(\hat y;t,q))\mid t,q\right].
\]

By the definition of conditional entropy for a binary random variable,
\[
\E\!\left[h_2(\eta(\hat y;t,q))\mid t,q\right]
=
H(y\mid \hat y,t,q).
\]
Therefore,
\begin{align}
p_c(t,q)
\le
\frac12\,H(y\mid \hat y,t,q).
\label{eq:pc_entropy_bound}
\end{align}

By definition of conditional mutual information,
\[
I(\hat y;y\mid t,q)
=
H(y\mid t,q)-H(y\mid \hat y,t,q).
\]
Rearranging gives
\[
H(y\mid \hat y,t,q)
=
H(y\mid t,q)-I(\hat y;y\mid t,q).
\]
Substituting into \eqref{eq:pc_entropy_bound}, we obtain
\[
p_c(t,q)
\le
\frac12\Big(H(y\mid t,q)-I(\hat y;y\mid t,q)\Big).
\]

Taking expectation over $(t,q)$ yields
\begin{align*}
p_{\mathrm{error}}
&=
\E_{t,q}\big[p_c(t,q)\big] \\
&\le
\frac12\Big(H(y\mid t,q)-I(\hat y;y\mid t,q)\Big).
\end{align*}
Finally, since conditioning reduces entropy, $H(y\mid t,q)\le H(y)$, we conclude that
\[
p_{\mathrm{error}}
\le
\frac12\Big(H(y)-I(\hat y;y\mid t,q)\Big).
\]
This completes the proof.
\end{proof}

\subsection{Proof of proposition \ref{pro}}

\begin{proposition}[Improved]
If $r_{\mathrm{analytic}}(t_1) \geq r_{\mathrm{analytic}}(t_2)$, and the conditional entropy given incorrect predictions does not increase too rapidly with $p(t,q)$, then
\[
H(y \mid \hat{y}, t_1, q) \leq H(y \mid \hat{y}, t_2, q).
\]
\end{proposition}

\begin{proof}
Let
\[
p_i := p(t_i, q) = \Pr(\hat{y} = y \mid t_i, q),
\]
and define
\[
C_i := H(y \mid \hat{y}, t_i, q, \hat{y} \neq y),
\]
the expected conditional entropy when the prediction is incorrect.

From the law of total entropy,
\[
H(y \mid \hat{y}, t_i, q) = p_i \cdot H(y \mid \hat{y} = y, t_i, q) + (1 - p_i) \cdot C_i.
\]
Since $H(y \mid \hat{y} = y, t_i, q) = 0$ (the answer is known exactly when prediction is correct), we have
\[
H(y \mid \hat{y}, t_i, q) = (1 - p_i) C_i. \tag{1}
\]

Given $r_{\mathrm{analytic}}(t_1) \geq r_{\mathrm{analytic}}(t_2)$ implies $p_1 \geq p_2$, we have
\[
1 - p_1 \leq 1 - p_2.
\]

Let $C(p)$ denote the average conditional entropy given incorrect predictions as a function of 
$p = \Pr(\hat y = y \mid t,q)$. 
The total conditional entropy can be decomposed as
\[
H(y \mid \hat y, t, q) 
= (1 - p)\, C(p),
\]
since correct predictions contribute zero uncertainty and all remaining uncertainty 
arises from incorrect cases. 
Therefore, if $(1 - p)C(p)$ is non-increasing in $p$, increasing the success probability $p$ 
leads to a reduction in the overall conditional entropy.
 
This is equivalent to
\[
\frac{d}{dp}\big[(1-p)C(p)\big] = (1-p)C'(p) - C(p) \leq 0,
\]
or
\[
C'(p) \leq \frac{C(p)}{1-p}.
\]
This condition allows $C(p)$ to increase with $p$, but not too rapidly. Intuitively, a better plan (higher $p$) may still produce wrong answers, but their distribution over $y$ should not become much more dispersed.

From (A1) and $p_1 \geq p_2$, we have
\[
(1 - p_1) C(p_1) \leq (1 - p_2) C(p_2).
\]
Substituting into (1) gives
\[
H(y \mid \hat{y}, t_1, q) \leq H(y \mid \hat{y}, t_2, q).
\]

Finally, from the identity
\[
I(y; \hat{y} \mid t, q) = H(y \mid t, q) - H(y \mid \hat{y}, t, q),
\]
and noting that $H(y \mid t, q)$ is independent of the predictor $\hat{y}$, we obtain
\[
I(y; \hat{y} \mid t_1, q) \geq I(y; \hat{y} \mid t_2, q). \qedhere
\]
\end{proof}

\subsection{Experimental parameter setup}
\label{sec:setup}

We conducted all experiments on eight H200 GPUs. In the supervised fine-tuning (SFT) stage, we trained Qwen2.5-7B-Instruct for 3 epochs. In the reinforcement learning (RL) stage, we adopted the GRPO algorithm, with a global batch size of 128 and a micro batch size of 4 per GPU. During rollout, the model generated 12 samples per step, including 4 analytic plans, each corresponding to 3 Chain-of-Thought (CoT) reasoning trajectories. For generation, we set temperature = 1.0 and top-p = 1.0, while for validation we used temperature = 0.6, top-p = 0.95. The number of RL training steps was configured as follows: LLaMA3.2-3B and Qwen2.5-7B-Instruct were trained for 350 steps, Qwen3-8B for 150 steps, and Qwen3-14B for 50 steps, with other hyperparameters kept the same across models. In addition, the learning rate (\texttt{lr}) was set to $1.0 \times 10^{-6}$, the weight decay (\texttt{weight\_decay}) was $1.0 \times 10^{-2}$, the optimizer was \texttt{adamw} (choices: \texttt{adamw} or \texttt{adamw\_bf16}), the learning‑rate warmup ratio (\texttt{lr\_warmup\_ratio}) was 0. The hyperparameter $\beta$ is assigned to 0.5.  For all Qwen3-8b, max token is 4.5k and for Qwen2.5-7B-Instrct , the max token is 3.5K.

\subsection{Test case of \alg}
\label{sec: test cases.}
\clearpage

\begin{figure*}
\begin{promptpbox}{Generated Samples in Real Test cases of \alg (Sample 1)}

\textbf{Question:} Determine the interval of convergence for the series  
\[
1+\frac{x}{2}+\frac{2! x^2}{3^2}+\frac{3!x^3}{4^3}+\ldots
\]

\begin{innerbox}[blues]{\texttt{\textless plan\textgreater}}
1. Identify the general term of the series.  
2. Apply the Ratio Test to find the radius of convergence.  
3. Check the endpoints of the interval separately to determine convergence or divergence there.  
4. State the final interval of convergence.  
\end{innerbox}

\begin{innerbox}[green]{\texttt{\textless think\textgreater}}
We start by writing down the general term of the series. The pattern is clear: the $n$-th term looks like

$$
\frac{n!}{(n+1)^n} x^n .
$$

To test for convergence, the Ratio Test is natural. Taking the ratio $a_{n+1}/a_n$, most factors cancel, and it reduces to

$$
|x|\left(\frac{n+1}{n+2}\right)^{n+1}.
$$

As $n \to \infty$, the power term tends to $e^{-1}$. So the limit is $|x|/e$. By the Ratio Test, convergence requires $|x|/e < 1$, which means $|x| < e$. That gives us the radius of convergence $R = e$.

Now check the endpoints. At $x = e$, the terms are $\frac{n!}{(n+1)^n} e^n$. Stirling’s formula shows that $n! \sim (n/e)^n \sqrt{2\pi n}$, so the terms behave like $\sqrt{n}$ times a constant. They don’t go to zero, so the series diverges. At $x = -e$, the alternating sign doesn’t help because the size of the terms again grows like $\sqrt{n}$. They still don’t go to zero, so the series diverges there as well.

Therefore, the series converges only strictly inside the interval, and the interval of convergence is

$$
(-e, e).
$$

\end{innerbox}

\textbf{Answer:}  
\[
\boxed{(-e,\,e)}
\]
\end{promptpbox}
\end{figure*}

\begin{figure*}
\begin{promptpbox}{Generated Samples in Real Test cases of \alg (Sample 2)}

\textbf{Question:} What is the least possible value of the sum $|x - 1| + |x - 1.5| + |x - 2|$?

\begin{innerbox}[blues]{\texttt{\textless plan\textgreater}}

1. Recognize that this is a minimization problem involving sums of absolute values.
2. Recall that for a sum of absolute values of linear functions, the minimum occurs at a median of the points.
3. Identify the points: 1, 1.5, and 2.
4. Determine the median of these three points.
5. Calculate the sum at the median point to find the minimum value.
\end{innerbox}

\begin{innerbox}[green]{\texttt{\textless think\textgreater}}
We want to minimize

$$
f(x) = |x - 1| + |x - 1.5| + |x - 2|.
$$

General fact: For $|x - a_1| + \cdots + |x - a_n|$, the minimum occurs when $x$ is the median of the $a_i$.

Here, the three points are $1, 1.5, 2$. Ordered: $1 < 1.5 < 2$.
So the median is $1.5$.

Evaluate at $x = 1.5$:

$$
|1.5 - 1| + |1.5 - 1.5| + |1.5 - 2| = 0.5 + 0 + 0.5 = 1.
$$

Check at neighbors:

* At $x = 1$: $0 + 0.5 + 1 = 1.5$.
* At $x = 2$: $1 + 0.5 + 0 = 1.5$.
* At $x = 1.25$: $0.25 + 0.25 + 0.75 = 1.25$.
* At $x = 1.75$: $0.75 + 0.25 + 0.25 = 1.25$.

All larger than 1.
Thus, the minimum occurs uniquely at $x = 1.5$, with value $1$.
\end{innerbox}

\textbf{Answer:}

$$
\boxed{1}
$$

\end{promptpbox}
\end{figure*}

\begin{figure*}
\begin{promptpbox}{Generated Samples in Real Test cases of \alg (Sample 3)}

\textbf{Question:} Find the range of $p$ for which the stability condition $|\xi| \leq 1$ holds, given that

$$
|\xi|^2 = 1 - 4p^2(1-p^2)\, s^4, \quad s = \sin\left(\tfrac{\omega}{2}\right) \in [0,1].
$$

\begin{innerbox}[blues]{\texttt{\textless plan\textgreater}}

1. Understand the stability condition: $|\xi| \leq 1$.
2. Substitute the given expression for $|\xi|^2$.
3. The stability condition $|\xi| \leq 1$ is equivalent to $|\xi|^2 \leq 1$.
4. Since $|\xi|^2$ is expressed as $1 - 4p^2(1-p^2)s^4$, ensure it is non-negative ($\geq 0$) to avoid complex values.
5. Analyze the sign of $p^2(1-p^2)$.
6. Determine the condition on $p$ such that $|\xi|^2 \in [0,1]$ for all $s \in [0,1]$.
7. Conclude the range of $p$.
\end{innerbox}

\begin{innerbox}[green]{\texttt{\textless think\textgreater}}
We have

$$
|\xi|^2 = 1 - 4p^2(1-p^2)s^4, \quad s \in [0,1].
$$

For stability, we need $|\xi| \leq 1$, i.e., $|\xi|^2 \leq 1$.
That means

$$
1 - 4p^2(1-p^2)s^4 \leq 1.
$$

This simplifies to

$$
-4p^2(1-p^2)s^4 \leq 0 \quad \Rightarrow \quad p^2(1-p^2) \geq 0.
$$

Since $p^2 \geq 0$, this requires $1 - p^2 \geq 0 \implies p^2 \leq 1$.
Thus $|p| \leq 1$.

Additionally, we must ensure $|\xi|^2 \geq 0$. For $|p| \leq 1$, the term $4p^2(1-p^2)s^4 \geq 0$, so $|\xi|^2 = 1 - \text{(nonnegative)} \in [0,1]$.
This guarantees both realness and the upper bound condition.

If $|p| > 1$, then $1 - p^2 < 0$, so $p^2(1-p^2) < 0$. In this case the subtracted term becomes negative, effectively making $|\xi|^2 > 1$ for some $s$, which breaks stability.

Hence the stability condition holds exactly when

$$
p \in [-1,1].
$$

\end{innerbox}

\textbf{Answer:}

$$
\boxed{[-1,1]}
$$

\end{promptpbox}
\end{figure*}

\begin{figure*}
\begin{promptpbox}{Example of an LLM correcting an incorrect plan through self-reflection.}

\textbf{Question:} Six points $ A, B, C, D, E, $ and $ F $ lie in a straight line in that order. Suppose that $ G $ is a point not on the line and that $ AC = 26 $, $ BD = 22 $, $ CE = 31 $, $ DF = 33 $, $ AF = 73 $, $ CG = 40 $, and $ DG = 30 $. Find the area of $ \triangle BGE $.

\begin{innerbox}[blues]{\texttt{\textless plan\textgreater}}
1. Assign coordinates to points A through F on the x-axis with A at 0.
2. Use the given segment lengths (AC, BD, CE, DF, AF) to set up equations and solve for all individual segment lengths AB, BC, CD, DE, EF.
3. Determine the coordinates of B, C, D, E, F.
4. Use distances CG and DG to locate point G by solving a system of equations from the distance formula.
5. Calculate area of triangle BGE using the shoelace formula with coordinates B, G, E. \textcolor{red}{\textbackslash \textbackslash
 Wrong plan}
\end{innerbox}

\begin{innerbox}[green]{\texttt{\textless think\textgreater}}
Wait, I notice the plan is incorrect in its approach to finding point G. The plan says to use distances CG and DG to locate G, but this will only give us the x-coordinate of G, not both coordinates. Let me revise this. \textcolor{red}{\textbackslash \textbackslash  
 Revise it by LLM self-reflection}

Looking more carefully at the problem, I need to find the area of triangle BGE. The points B, G, and E form a triangle where:
- B and E are on the x-axis
- G is somewhere in the plane

The area of triangle BGE can be calculated as ½ × |BE| × height from G to line BE. Since BE is horizontal (both B and E on x-axis), the height is simply |y-coordinate of G|.

So the revised approach should be:
1. Find coordinates of B and E
2. Find the y-coordinate of G
3. Area = $\frac{1}{2} \times |BE| × |y_G|$.

Let me implement this revised plan.

First, let's find all segment lengths as in the original plan:
Let AB = a, BC = b, CD = c, DE = d, EF = e

From AC = 26: a + b = 26
From BD = 22: b + c = 22  
From CE = 31: c + d = 31
From DF = 33: d + e = 33
From AF = 73: a + b + c + d + e = 73

From a + b = 26 and a + b + c + d + e = 73, we get c + d + e = 47
From c + d = 31 and c + d + e = 47, we get e = 16
From d + e = 33 and e = 16, we get d = 17
From c + d = 31 and d = 17, we get c = 14
From b + c = 22 and c = 14, we get b = 8
From a + b = 26 and b = 8, we get a = 18

So: AB = 18, BC = 8, CD = 14, DE = 17, EF = 16

Place A at (0,0), then:
B = (18, 0)
C = (26, 0) 
D = (40, 0)
E = (57, 0)
F = (73, 0)

Now, to find the y-coordinate of G:
We know CG = 40 and DG = 30
Let G = (x,y)

CG² = (x-26)² + y² = 1600
DG² = (x-40)² + y² = 900

Subtract: (x-26)² - (x-40)² = 700
(x² - 52x + 676) - (x² - 80x + 1600) = 700
28x - 924 = 700
28x = 1624
x = 58

Now use (x-26)² + y² = 1600:
(58-26)² + y² = 1600
32² + y² = 1600
1024 + y² = 1600
y² = 576
y = 24 (taking positive)

So G = (58, 24)

Now triangle BGE:
B = (18, 0), E = (57, 0), G = (58, 24)
BE = 57 - 18 = 39
Height = 24 (y-coordinate of G)

Area = ½ × 39 × 24 = 468

\end{innerbox}

\textbf{Answer:}

$$
\boxed{468}
$$

\end{promptpbox}
\end{figure*}

\clearpage

\begin{figure*}
\begin{rlpbox}[\alg prompt]
\{\{ content | trim \}\}

Begin by outlining a high-level plan wrapped in <plan></plan> tags.

- This plan should capture only the major phases, strategic choices, and conditional branches.

- Avoid low-level steps, calculations, or detailed reasoning here.
Next, reason step by step within <think></think>.

- During reasoning, critically evaluate the initial plan. If you find any errors, inconsistencies, or improvements needed, revise your plan mentally and continue reasoning based on the revised plan.

- Explicitly state if you are revising the plan and describe the changes.

- This is your detailed chain-of-thought: work through assumptions, intermediate steps, and logical derivations until the solution is reached.  

Finally, provide the final answer enclosed within 
\begin{lstlisting}
\boxed{}
\end{lstlisting}

\end{rlpbox}
\end{figure*}
%%%%%%%%%%%%%%%%%%%%%%%%%%%%%%%%%%%%%%%%%%%%%%%%%%%%%%%%%%%%%%%%%%%%%%%%%%%%%%%
%%%%%%%%%%%%%%%%%%%%%%%%%%%%%%%%%%%%%%%%%%%%%%%%%%%%%%%%%%%%%%%%%%%%%%%%%%%%%%%

\end{document}